%% file: preprint.tex
\documentclass[10pt]{article} 

\usepackage{booktabs}       
\usepackage{amsfonts}       
\usepackage{nicefrac}       
\usepackage[final]{microtype}      
\usepackage[dvipsnames,table]{xcolor}         
\usepackage{graphicx}
\usepackage{subcaption}
\usepackage{titletoc}
\usepackage{wrapfig}

\usepackage[p,osf,scaled=.98,space]{erewhon}
\usepackage[varqu,varl]{inconsolata} 
\usepackage[type1,scaled=.95]{cabin} 
\usepackage[T1]{fontenc}    

\usepackage[
  a4paper,
  top    = 33.3mm,
  bottom = 53.9mm,
  left   = 30.8mm,
  right  = 30.8mm,
]{geometry}

\usepackage{parskip}
\usepackage[backref=true,backend=biber,style=authoryear-comp,sortcites,sorting=ynt,maxbibnames=99,maxcitenames=1,uniquelist=true,uniquename=init,giveninits]{biblatex}
\DeclareNameAlias{sortname}{family-given}

\usepackage{titletoc}
\contentsmargin{2.55em}
\dottedcontents{lsection}[3.8em]{\small \itshape}{2.3em}{1pc}
\dottedcontents{lsubsection}[6.1em]{\small \itshape}{2.3em}{1pc}
\titlecontents*{lsubsubsection}[6.1em]{\footnotesize \itshape}{}{}{}[ \textbullet\ ][]

\usepackage{amsmath}
\usepackage{amsthm}
\usepackage{amssymb}
\usepackage{mathtools}
\usepackage{multirow}
\usepackage{tabularx}
\usepackage{enumitem}
\usepackage{bm}
\usepackage{mathrsfs}

\usepackage{algorithm}
\usepackage{algpseudocode}

\usepackage{pifont}

\newcommand*{\R}{\mathbb{R}}

\newcommand*{\T}{\mathbb{T}}

\newcommand*{\bx}{{\bm{x}}}
\newcommand*{\bX}{{\bm{X}}}

\newcommand*{\bW}{{\bm{W}}}

\newcommand*{\bu}{{\bm{u}}}

\newcommand*{\bmf}{{\bm{f}}}
\newcommand*{\bg}{{\bm{g}}}

\newcommand*{\rmd}{\mathrm{d}}

\DeclareFontFamily{U}{mathx}{}
\DeclareFontShape{U}{mathx}{m}{n}{<-> mathx10}{}
\DeclareSymbolFont{mathx}{U}{mathx}{m}{n}
\DeclareMathAccent{\widecheck}{0}{mathx}{"71}

\definecolor{tufte-red}{HTML}{990000}    

\definecolor{tufte-dark}{HTML}{111111}   
\definecolor{tufte-rule}{HTML}{AAAAAA}   
\definecolor{tufte-shade}{HTML}{F2F2F0}  
\definecolor{tufte-mid}{HTML}{666666}    

\definecolor{tufte-green}{HTML}{2E6B3E}  

\definecolor{tufte-sienna}{HTML}{8B4513} 
\definecolor{tufte-sienna-lt}{HTML}{C8A882} 

\definecolor{w-crimson} {HTML}{990000}   
\definecolor{w-sienna}  {HTML}{8B4513}   
\definecolor{w-sienna2} {HTML}{C8A882}   
\definecolor{w-rust}    {HTML}{B5541C}   
\definecolor{w-mid}     {HTML}{888888}   
\definecolor{w-navy}    {HTML}{1B2A6B}   

\usepackage{pgfplots}
\usepackage[edges]{forest}
\pgfplotsset{compat=newest}
\usetikzlibrary{trees,positioning,calc,backgrounds}
\usetikzlibrary{positioning, arrows.meta, shapes.geometric, fit,
                backgrounds, calc, decorations.pathmorphing, patterns}

\pgfplotsset{
    nfeplot/.style={
        axis background/.style={fill=gray!10},
        width=\linewidth,
        height=6.5cm,
        grid=both,
        grid style={white},
        tick style={thick, white},
        major grid style={white},
        minor grid style={gray!10, dashed},
        minor x tick num=9,
        minor y tick num=0,
        xlabel={Number of Function Evaluations (NFE) $(\downarrow)$},
        ylabel={#1},
        legend pos=north east,
        legend cell align=left,
        legend style={
            font=\scriptsize,
            fill opacity=0.85,
            draw opacity=1,
            text opacity=1,
            text=black,
        },
        tick label style={font=\scriptsize},
        label style={font=\scriptsize},
        thick,
        mark size=2.5pt,
    },
    mtwoforone/.style={draw=w-crimson, mark=*,
                       mark options={fill=white}},
    mdiffusion/.style={draw=w-sienna,  mark=*,
                       mark options={fill=white}},
    mmixture/.style  ={draw=w-sienna2, mark=*,
                       mark options={fill=white}},
    mfp/.style       ={draw=w-rust,    mark=*,
                       mark options={fill=white}},
    mboth/.style     ={draw=w-mid,     mark=*,
                       mark options={fill=white}},
    mssfm/.style     ={draw=w-navy,    solid, mark=*,
                       mark options={fill=white}},
}

\newcommand{\cmark}{\textcolor{tufte-green}{\ding{51}}}
\newcommand{\xmark}{\textcolor{tufte-red}{\ding{55}}}

\PassOptionsToPackage{hyphens}{url}\usepackage[breaklinks=true,bookmarks=false]{hyperref}
\hypersetup{
    colorlinks,
    linkcolor=tufte-red,
    citecolor=tufte-red,    
    urlcolor=tufte-red
}
\usepackage[capitalise, noabbrev]{cleveref}
\crefname{assumption}{Assumption}{Assumptions}
\crefname{appendix}{Appendix}{Appendices}

\usepackage{xspace}
\makeatletter
\DeclareRobustCommand\onedot{\futurelet\@let@token\@onedot}
\def\@onedot{\ifx\@let@token.\else.\null\fi\xspace}

\def\eg{\emph{e.g}\onedot} 
\def\ie{\emph{i.e}\onedot}

 \def\vs{\emph{vs}\onedot}
\def\wrt{w.r.t\onedot} 
\makeatother

\usepackage[cachedir=minted-cache]{minted}

\usepackage[many]{tcolorbox}
\tcbuselibrary{listings, minted, skins}

\newtcolorbox{standoutbox}{
    enhanced,
    sharp corners,
    breakable,
    boxrule=1pt,
    notitle,
    colback=ctp-mantle,
    colframe=ctp-blue,
}

\usepackage{thmtools}

\declaretheoremstyle[
    headfont=\color{tufte-dark}\bfseries, 
    bodyfont=\itshape,
    headpunct={.},
    postheadspace=0.5em,
]{thmstyle}

\declaretheoremstyle[
    headfont=\color{tufte-dark}\bfseries,
    bodyfont=\normalfont,
    headpunct={.},
    postheadspace=0.5em,
]{defstyle}

\declaretheorem[name=Assumption,parent=section]{assumption}

\tcolorboxenvironment{restatable}{
    enhanced,
    breakable,
    boxrule=0pt,
    arc=6pt,
    colback=gray!10,                    
    colframe=gray!10,
}

\tcolorboxenvironment{theorem}{
    enhanced,
    arc=6pt,
    breakable,
    boxrule=0pt,
    colback=gray!10,                    
    colframe=gray!10,
}

\newtcolorbox{theorembox}{
    enhanced,
    arc=6pt,
    breakable,
    boxrule=0pt,
    colback=gray!10,                    
    colframe=gray!10,
}

\tcolorboxenvironment{definition}{
    enhanced,
    arc=6pt,
    breakable,
    boxrule=0pt,
    colback=white,
    colframe=white,
}

\tcolorboxenvironment{lemma}{
    enhanced,
    arc=6pt,
    breakable,
    boxrule=0pt,
    colback=white,
    colframe=white,
}

\newtcbinputlisting[use counter=filePrg, number within=section, list inside=mypyg]{\codeFromFile}[4]{%
    listing engine=minted,
    minted language=#1,
    listing file={#2},
    minted style=gruvbox-light,
    minted options={autogobble, tabsize=2, fontsize=\tiny, breaklines},
    listing only,
    enhanced,
    sharp corners,
    boxrule=0.5pt,
    colback=gray!10,
    colframe=gray,
    label=lst:#4
 }

 \crefname{filePrg}{Code}{Codes}

\makeatletter 
\renewcommand{\l@tcolorbox}{\@dottedtocline{1}{0pt}{2.3em}}
\makeatother

\usepackage{titlesec}

\titleformat{\section}[hang]{\Large \sffamily \bfseries}{\thesection}{1em}{}
\titleformat{\subsection}[hang]{\large \sffamily \bfseries}{\thesubsection}{1em}{}
\titleformat{\subsubsection}[hang]{\normalsize \sffamily \bfseries}{\thesubsubsection}{1em}{}
\titleformat{\paragraph}[runin]{\normalfont\normalsize\sffamily\bfseries}{\theparagraph}{1em}{}

\title{Strong Stochastic Flow Maps}
\date{}

\begin{document}

\begin{tcolorbox}[
    colback=gray!10,
    colframe=gray!10,
    boxrule=0pt,
    arc=12pt,
]

\vspace{1em}

\noindent{\rule{\linewidth}{2pt}}

\vspace{1em}

\begin{center}
    {\sffamily\bfseries\huge Strong Stochastic Flow Maps}
\end{center}

\vspace{0.5em}

\noindent{\rule{\linewidth}{1pt}}

\vspace{1em}

\begin{center}
    {\sffamily\bfseries
        Sam McCallum\textsuperscript{*\,1}\quad
        Zander W. Blasingame\textsuperscript{*\,2}\quad
        Timothy Herschell\textsuperscript{1}\quad
        Niklas Rindtorff\textsuperscript{2}\\
        Alexander Tong\textsuperscript{\textdagger\,2}\quad
        James Foster\textsuperscript{\textdagger\,1}
    }\\[0.5em]
    {\sffamily
        \textsuperscript{1}University of Bath \quad
        \textsuperscript{2}AITHYRA\\
        
    }
\end{center}

\vspace{1em}

\begin{center}
    \begin{minipage}{\textwidth}
    \small\sffamily
  Flow and diffusion models generate high-quality samples in many modalities; however, many network evaluations are required during inference due to numerical integration of an underlying differential equation. Flow maps alleviate this problem by learning the solution map of the differential equation directly, enabling few-step sampling. Yet, current methods are restricted to approximating the solution map of ODEs. These methods can be used to learn the transition kernel of an SDE, thereby obtaining a solution map that recovers the marginal distributions of the process (weak convergence) rather than the solution path (strong convergence). We propose \textsc{Strong Stochastic Flow Maps} (SSFMs) as a novel framework for learning the \textit{strong} solution map of additive-noise SDEs, directly generalizing deterministic flow maps to the stochastic setting. Further, a polynomial approximation to Brownian motion is introduced and shown to converge pathwise. These results enable a simulation-free training objective for the solution map of diffusion models. We demonstrate that SSFMs outperform previous stochastic flow map methods on image generation and enable few-step sampling of molecular systems.
    \end{minipage}
\end{center}

\vspace{1em}

{\small\sffamily
    \textbf{Code:} \url{https://github.com/sammccallum/ssfm}\\
    \textbf{Correspondence:}~\href{mailto:sm2942@bath.ac.uk}{\texttt{sm2942@bath.ac.uk}}
    and
    \href{mailto:zblasingame@aithyra.at}{\texttt{zblasingame@aithyra.at}}\\
    \textbf{Date:} 31 May, 2026 \\
    {\small\itshape
        \textsuperscript{*}Equal contribution \quad \textsuperscript{\textdagger}Equal supervision
    }
}

\hfill\includegraphics[height=0.25in]{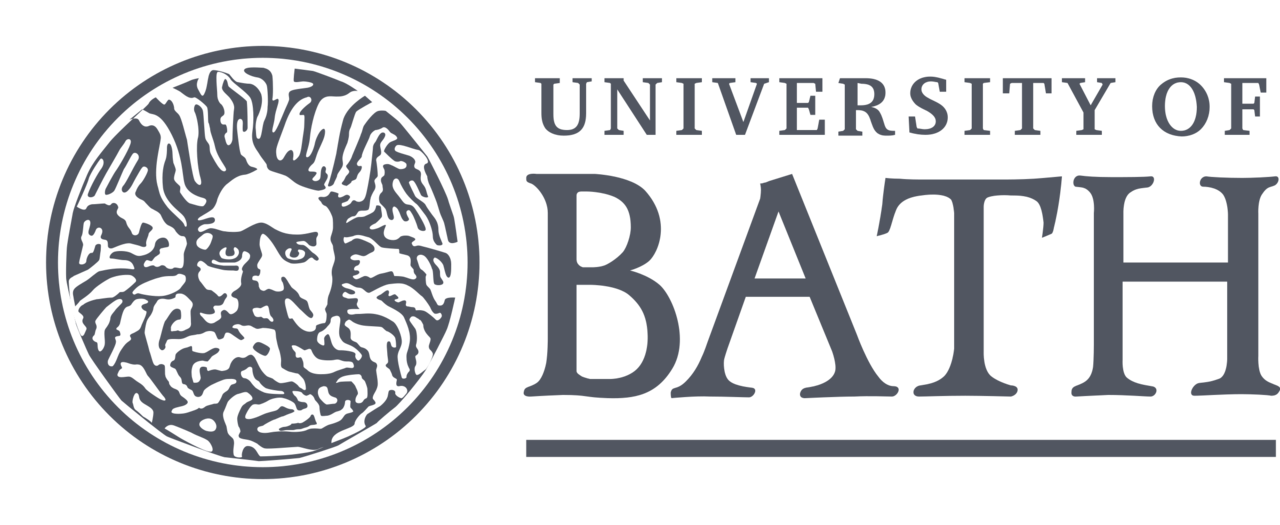}
\qquad
\includegraphics[height=0.25in]{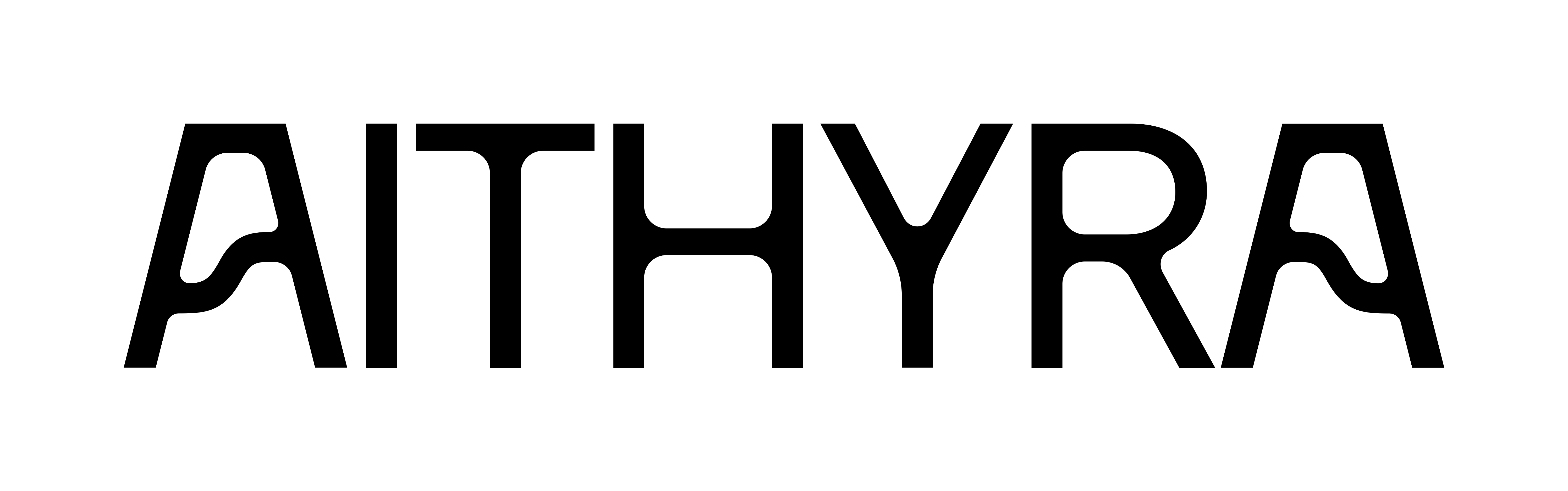}

\end{tcolorbox}

\section{Introduction}
Deep generative models based on \textit{neural differential equations} \parencite{chen2018neural,kidger_thesis} have become one of the most successful model classes for solving a variety of problems such as generative modeling \parencite{song2021scorebased,liu2023flow} and time-series data \parencite{oh2024stable,walker2024log}. 
The application of neural differential equations to generative modeling---diffusion models \parencite{song2021scorebased,ho2020diffusion,sohl2015deep} and flow matching models \parencite{peluchetti2021,lipman2023flow,liu2023flow}---have become state-of-the-art for many different tasks
including
image \parencite{rombach2022high} and video \parencite{blattmann2023align} generation,
protein design \parencite{watson2023novo,disco2026}, and Boltzmann sampling \parencite{rehman2026falcon}.

While quite expressive generative models, neural differential equations often require a large \textit{number of function evaluations} (NFEs) of the learned vector fields to integrate the underlying differential equation.
As such there has been great interest in learning how to improve the computational efficiency of these models, \eg, proposing better numerical schemes to reduce the number of NFEs whilst maintaining similar performance
\parencite{lu2022dpmsolver,zhang2023fast,gonzalez2023seeds}.
Recently, another direction has looked at how to learn the solution, or flow, map associated with a neural \textit{ordinary differential equation} (ODE) \parencite{song2023consistency,kim2024consistency,heek2024multistep,liu2023flow,boffi2024flow,geng2025mean,sabour2025align}.
These methods, which learn a \textit{neural flow map}, have obtained state-of-the-art performance with low NFEs in image generation  \parencite{geng2025mean,geng2025improved} and Boltzmann sampling \parencite{rehman2026falcon} compared to the 100s of steps required with diffusion models.

Previous work has focused on learning a \textit{deterministic} map from the source noise to the target distribution. Recent work has extended these deterministic maps to the stochastic setting by approximating the transition distributions of a stochastic process \parencite{holderrieth2026diamond,potaptchik2026meta,passaro2026stochastic}. However, such approaches do not allow for the estimation of pathwise observables as they are fundamentally decoupled from the underlying stochastic differential equation (SDE) and are only capable of computing \emph{weak} (convergence in distribution) solutions of the stochastic process \parencite{oksendal2003stochastic}.

We introduce \textsc{Strong Stochastic Flow Maps} (SSFMs), a novel framework which obtains the \emph{strong} (convergence in path) solution map to additive-noise SDEs.
This naturally extends ODE-based flow maps, which are pathwise solution maps, to the stochastic setting. 
Specifically, given a realization of the Brownian path and an initial condition, we learn the pathwise solution map of the additive-noise SDE.
To efficiently implement such a model we require a novel set of tools which we construct in this work. Our main contributions are summarized as follows:

\begin{itemize}[topsep=0pt, partopsep=0pt, itemsep=3pt, parsep=0pt, leftmargin=*]
    \item We develop a novel framework for learning the strong solution to additive-noise SDEs in contrast with prior works which consider weak approximations.
    \item We formalize a training objective for learning this map which admits a simulation-free algorithm for obtaining the strong solution map of diffusion models.
    \item We introduce a polynomial approximation of the Brownian motion and prove that it converges in $\alpha$-H\"older distance, has tractable coefficients, and admits closed form Chen relations.
    \item We demonstrate this formulation achieves state-of-the-art performance on few-step stochastic image generation and enables sampling of equilibrium molecular conformations with as few as 1-2 NFEs.
\end{itemize}

\begin{figure}[t]
    \centering
    \input{tikz/sfm_s_to_t.tex}
    \caption{%
        \textbf{Strong \vs\ weak stochastic flow maps.}
        \emph{Top}: The strong stochastic flow map solution is consistent pathwise for a given realization of the Brownian motion $\bW_t(\omega)$.
        \emph{Bottom}: The weak stochastic flow map samples independently from the marginal distribution at each time.}
    \label{fig:weak_vs_strong}
\end{figure}
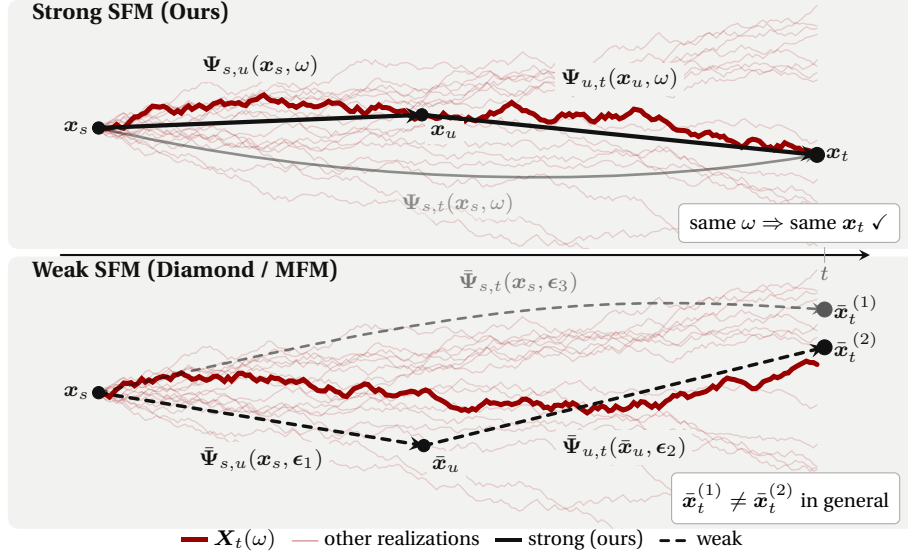

\section{Background and related work}
\paragraph{Neural SDEs.}
Neural SDEs are the stochastic extension of neural ODEs \parencite{chen2018neural}, independently proposed by \textcite{li2020sdes,kidger2021efficient} which aim to learn the drift and diffusion coefficients for an arbitrary SDE.
In this work we will focus specifically on the case of additive-noise SDEs with known diffusion coefficients, given by
\begin{equation}
    \rmd \bX_t = \bmf^\theta(t, \bX_t)\;\rmd t + g(t)\;\rmd \bW_t,
\end{equation}
where $\bmf^\theta$ is the drift coefficient we aim to learn. Previous work has focused on learning these coefficients by integrating the SDE numerically and then performing backpropagation through the numerical solution. However, for many SDEs studied in generative modelling we can learn the drift coefficient via \textit{simulation-free} training which has underpinned the popularity of different \textit{matching} objectives such as score matching and flow matching. In this work, we focus on diffusion SDEs; however, our results hold for any additive-noise SDE.

\begin{wraptable}{r}{0.5\textwidth}
    \vspace{-1em}
    \centering
    \caption{Comparison of flow-based models.}
    \label{tab:method-comparison}
    \small
    \begin{tabular}{lccc}
        \toprule
        Method & \rotatebox{45}{Few-step} & \rotatebox{45}{Pathwise} & \rotatebox{45}{Stochastic}\\
        \midrule
        Flow models          & \xmark & \cmark & \xmark \\
        Diffusion models     & \xmark & \cmark & \cmark \\
        GLASS flows          & \xmark & \xmark & \cmark \\
        Flow maps            & \cmark & \cmark & \xmark \\
        Weak SFMs            & \cmark & \xmark & \cmark \\
        \textbf{SSFM (ours)} & \cmark & \cmark & \cmark \\
        \bottomrule
    \end{tabular}
    \vspace{-1em}
\end{wraptable}

\paragraph{Flow matching and diffusion.}
In this section we introduce the necessary background on flow/diffusion models \parencite{peluchetti2021,lipman2023flow,albergo2025stochastic,liu2023flow,tong2024improving}. We denote data samples as $\bX_1 \in \mathbb{R}^d$ drawn from the data distribution with density $p_1 \equiv q(\bx)$. We take the source distribution to be a unit Gaussian with density $p_0 \equiv p(\bx)$ with $\bX_0 \sim p_0$. Following \textcite{lipman2024flow-guide} we consider the scenario of \textit{affine Gaussian probability paths} where we define a random variable $\bX_t \coloneq \alpha_t \bX_1 + \sigma_t \bX_0$ with noise schedule $(\alpha_t, \sigma_t)$, where $\alpha_t, \sigma_t \geq 0$ with $\alpha_0 = \sigma_1 = 0$, $\alpha_1 = \sigma_0 = 1$, and $\alpha_t$ ($\sigma_t$ resp.) is strictly monotonically increasing (decreasing resp.) and continuously differentiable.

In the flow matching framework we learn the marginal vector field as
\begin{equation}
\bu_t(\bX) = \int_{\R^d} \bu_{t|1}(\bX|\bX_1) p_{1|t}(\bX_1|\bX)\; \rmd \bX_1, \qquad p_{1|t}(\bX_1|\bX) = \frac{p_{t|1}(\bX|\bX_1)p_1(\bX_1)}{p_t(\bX)},
\end{equation}
where $\bu_{t|1}(\cdot | \bX_1)$ is the vector field conditioned on a data sample $\bX_1$. This vector field can be shown to satisfy,
\begin{equation}
    \label{eq:fm_ode}
    \bX_0 \sim p_0, \quad \bX_0 + \int_0^t \bu_s(\bX_s)\;\rmd s = \bX_t \sim p_t,
\end{equation}
such that the solution to \eqref{eq:fm_ode} at time $t = 1$ yields samples from the data distribution $q \equiv p_1$.

The ODE in \cref{eq:fm_ode} is referred to as the \textit{probability flow ODE} \parencite{song2021scorebased} and can be written as an SDE with the same marginal distributions, given by
\begin{equation}
    \label{eq:fm_sde_form}
    \rmd \bX_t = \left[2 \bu_t(\bX_t) - \frac{\dot\alpha_t}{\alpha_t}\bX_t\right]\;\rmd t + \nu_t\; \rmd \bW_t, \quad \nu_t^2 = 2 \frac{\dot\alpha_t}{\alpha_t}\sigma_t^2 - 2\sigma_t\dot\sigma_t,
\end{equation}
where $\bW_t$ is the standard Brownian motion. This derivation follows straightforwardly from \textcite{anderson1982reverse} and has been discussed more recently in the context of diffusion models \parencites{holderrieth2026glass,song2021scorebased, maoutsa2020interacting}. The implications for model performance when sampling from either the SDE or ODE formulation with equivalent marginals is discussed in \parencite{albergo2025stochastic, nie2024blessing}.

\paragraph{Flow maps.}
Since integrating the generative models in \eqref{eq:fm_ode} or \eqref{eq:fm_sde_form} requires many function evaluations, recent work has proposed to instead learn the integral, or solution map, directly at training time. These works include consistency models \parencite{song2023consistency,heek2024multistep,kim2024consistency}, shortcut models \parencite{frans2025one}, mean flows \parencite{geng2025mean}, and flow maps \parencite{boffi2024flow,boffi2025build,sabour2025align}.
There exists a variety of techniques for training such models. However, training can largely be broken down into two loss components: one term learns the instantaneous behaviour at $s = t$ and another learns the flow map $t > s$. Training via such an objective is referred to as self-distillation \parencite{boffi2025build}.

\paragraph{Weak stochastic flow maps.}
Recently, several works have explored a stochastic extension of deterministic flow maps \parencite{kiyohara2025neural,potaptchik2026meta,holderrieth2026diamond,passaro2026stochastic}, which can be characterized as learning the transition kernel $p_{t|s}(\bX_t|\bX_s)$ of some underlying SDE. Specifically, most of these methods \parencite{potaptchik2026meta,holderrieth2026diamond,passaro2026stochastic} proceed by defining a deterministic ODE flow map where each step of the map is defined by an inner flow model,\footnote{In contrast, \textcite{kiyohara2025neural} propose a variational objective.} given by
\begin{equation}
    \label{eq:glass_ode}
    \bX_s + \int_s^t \bar\bu_\tau(\bar\bX_\tau|\bX_s, s)\;\rmd \tau = \bar\bX_t \sim p_{t|s}(\cdot | \bX_s).
\end{equation}
Since these methods learn the transition kernel $p_{t|s}$ they are described as exhibiting \emph{weak} (in distribution) convergence \parencite{oksendal2003stochastic}. This is in contrast to the pathwise solution map of the underlying additive-noise SDE which generates $\bX_t$ given an initial condition $\bX_s$ and a realization of the Brownian motion $\{\bW_u\}_{s \leq u \leq t}$. We summarize the relationship between SSFMs and prior methods in \cref{tab:method-comparison}.

\section{Strong Stochastic Flow Maps}
\label{sec:ssfm}
In this section we describe the transition from deterministic flow maps to strong stochastic flow maps. We present a natural extension of deterministic flow maps by learning the solution map from both an initial condition and a realization of the Brownian path. We then show how this solution map can be obtained by minimization of an appropriate self-distillation objective.

\subsection{The It\^o map}
Consider the following additive-noise It\^o SDE,
\begin{equation}
    \label{eq:sde}
    \rmd \bX_t = \bm f(t, \bm X_t)\;\rmd t + \bm g(t)\;\rmd\bW_t, 
\end{equation}
where $\bX_t \in \mathbb{R}^d$ is a continuous-valued stochastic process with initial condition $\bX_0$, $\bmf : \mathbb{R} \times \mathbb{R}^d \rightarrow \mathbb{R}^d$ is the drift function, $\bg : \mathbb{R} \rightarrow \mathbb{R}^{d\times w}$ is the diffusion coefficient and $\bW_t \in \mathbb{R}^w$ is Brownian motion. It is assumed that $\bmf$ and $\bg$ are suitably regular such that a unique strong solution $\bX_t$ exists (see Assumption \ref{ass:unique}).

It is natural to consider the solution map of this SDE, $\bm\Psi_{s, t}: (\bX_s, \bW_{[s, t]}) \mapsto \bX_t$, where the Brownian path is written as $\bW_{[s, t]} = \{\bW_u\}_{s\leq u \leq t}$. This map $\bm\Psi_{s, t}$ is known as the It\^o map.\footnote{Note that for a general SDE with state-dependent diffusion, $\bg(t, \bX_t)$, the It\^o map is not a well-defined continuous function. This is solved by rough path theory and the It\^o-Lyons map \parencite{lyons1998differential}.} For an SDE with additive noise and suitably regular coefficients, the It\^o map is well-posed and continuous \parencite{friz2010multidimensional}. It is this map that we aim to approximate with a neural network, $\bm \Psi^\theta_{s, t}$.

Since the It\^o map is defined with respect to both an initial condition, $\bX_0$, and a realization of $\bW_{[s, t]}$, the convergence to the solution $\bX_t$ is pathwise. In stochastic analysis, a process approximating such a pathwise solution is called \emph{strongly} convergent \parencite{oksendal2003stochastic}. We therefore refer to our method as \emph{Strong Stochastic Flow Maps}.

\subsubsection{Constructing the It\^o map}
Analogously to the deterministic flow map, we can derive a tangent condition that the It\^o map satisfies. This result will allow us to associate the vector fields of an SDE with the It\^o map.

\begin{restatable}[Tangent Condition]{lemma}{tangent}
    \label{lem:tangent}
    Let $\bm\Psi_{s, t}(\bX_s, \bW_{[s, t]})$ denote the It\^o map for \eqref{eq:sde}. Then,
    \begin{equation}
        \label{eq:tangent}
        \lim_{s \rightarrow t} \rmd \bm\Psi_{s, t}(\bX_s, \bW_{[s, t]}) = \bm f(t, \bX_t)\;\rmd t + \bm g(t)\;\rmd \bW_t.
    \end{equation}
\end{restatable}

To construct the It\^o map we therefore propose an Euler-Maruyama step-like object that satisfies the tangent condition. We will show that such a parameterization attains the It\^o map when trained according to the objective in \eqref{eq:objective}.

\begin{restatable}[Strong Stochastic Flow Map]{proposition}{ssfm}
    \label{prop:em-sfm}
    Consider an Euler-Maruyama parameterization of the It\^o map,
    \begin{equation}
        \label{eq:em-sfm}
        \bm\Psi_{s, t}(\bX_s, \bW_{[s, t]}) = \bX_s + \bm f_{s, t}(\bX_s, \bW_{[s, t]})(t-s) + \bm g_{s, t}(\bW_{[s, t]})(\bW_t - \bW_s),
    \end{equation}
    where $\bm f, \bm g$ are twice continuously differentiable, Lipschitz in both time arguments, satisfy the conditions that $\bm f_{t, t}(\bX_t, \bW_{[s, t]}) = \bm f_{t, t}(\bX_t)$, and $\bm g_{t, t}(\bW_{[s, t]}) = \bm g_{t, t}$ (i.e. the coefficients are independent of $\bW_{[s, t]}$ for $s=t$). Then, $\bm\Psi_{s, t}$ satisfies the tangent condition \eqref{eq:tangent} if and only if
    \begin{equation*}
        \bm f_{t, t}(x_t) = \bm f(t, x_t), \quad \bm g_{t, t} = \bm g(t).
    \end{equation*}
\end{restatable}

The proposed stochastic flow map in Proposition \ref{prop:em-sfm} introduces a drift $\bm f_{s, t}$ and a diffusion $\bm g_{s, t}$. The intuition behind these functions is that they act as the normalized drift and diffusion integrals, respectively. Therefore, they must depend on the underlying driving path $\bW_{[s, t]}$.

We see from Proposition \ref{prop:em-sfm} that for the stochastic flow map to satisfy the tangent condition, the drift integral must collapse to the drift function, $\bm f_{t, t}(x_t) = \bm f(t, x_t)$, and the diffusion integral must collapse to the diffusion coefficient, $\bm g_{t, t} = \bm g(t)$, as $s \rightarrow t$. This indicates that $\bm f_{t, t}(x_t)$ and $\bm g_{t, t}$ can be estimated by the matching objective in \eqref{eq:objective}.

Next, we must establish an objective that constrains the finite-time, $t > s$, behaviour of the stochastic flow map. Consider the semigroup condition, 
\begin{equation}
    \label{eq:semigroup}
    \bm\Psi_{s, t}(\bX_s, \bW_{[s, t]}) = \bm\Psi_{u, t}(\bm\Psi_{s, u}(\bX_s, \bW_{[s, u]}), \bW_{[u, t]}),
\end{equation} 
for $s < u < t$. Note this condition is satisfied by the It\^o map (see the proof of \cref{prop:semigroup}). It is possible to show that if the strong stochastic flow map construction from Proposition \ref{prop:em-sfm} satisfies the tangent condition \eqref{eq:tangent} and the semigroup property, then this map \emph{is} the It\^o map. This is given by Proposition \ref{prop:semigroup}.

\begin{restatable}[Semigroup condition]{proposition}{semigroup}
    \label{prop:semigroup}
    Let $\bm \Psi_{s, t}(\bX_s, \bW_{[s, t]})$ denote the strong stochastic flow map satisfying \eqref{eq:tangent} and \eqref{eq:em-sfm}. Then $\bm\Psi_{s, t}(\bX_s, \bW_{[s, t]})$ is the It\^o map if and only if the semigroup property \eqref{eq:semigroup} holds.
\end{restatable}

\subsubsection{Constructing the Brownian motion}
Since the It\^o map requires a realization of the Brownian path $\bW_{[s, t]}$, so too do the drift and diffusion integral functions $\bm f_{s, t}, \bm g_{s, t}$. These functions will ultimately be parameterized as neural networks. We therefore require an efficient characterization of $\bW_{[s, t]}$ that we can pass as input.

A na\"ive approach is to simply pass a piecewise linear approximation to $\bW_{[s, t]}$ over $N$ intervals. However, to obtain any reasonable approximation to the path, too many intervals would be required. We instead look to pass the coefficients of a polynomial expansion of $\bW_{[s, t]}$. Specifically, we consider a piecewise polynomial expansion to Brownian motion in terms of shifted Legendre polynomials \parencite{foster2020optimal, Habermann_2021}.

This polynomial expansion has a number of desirable properties: (1) the coefficients of this expansion appear in the stochastic Taylor expansion of the SDE in \eqref{eq:sde} enabling larger time-step approximation; (2) the coefficients are independently and normally distributed allowing for tractable and exact sampling; (3) the coefficients admit closed-form Chen relations \parencite{chen1954iterated,chen1957integration} such that two coefficients over the sub-intervals $[s, u]$, $[u, t]$ combine into a single coefficient over $[s, t]$. This property enables the semigroup objective to be implemented. These results are shown in \cref{thm:properties_of_polynomial}. 

The following results are theoretically dense, relying on elements of rough path theory \parencite{lyons1998differential}. However, the polynomial expansion of Brownian motion is visually intuitive; see \cref{fig:poly_brownian}. 

\begin{restatable}[Polynomial approximation for Brownian motion]{definition}{poly_approx_brownian}
    \label{def:poly_approx_brownian}
    Let $\widetilde{P}_n: [0,1] \to \R$ denote the $n$-th shifted Legendre polynomial on $[0,1]$. 
    We define the $n$-th coefficient of the polynomial expansion as,
    \begin{equation}
        \bm I_{s,t}^{(n)} = \int_s^t  \widetilde{P}_n\left(\frac{u-s}{t-s}\right) \rmd \bW_u.
    \end{equation}
    As introduced in \parencite{foster2020optimal, Habermann_2021}, the degree-$N$ polynomial approximation of the Brownian motion on $[s,t]$ takes the form,
    \begin{equation}
        \label{eq:polyapprox}
        \bW_{u,v}^{(N)} = \sum_{n=0}^{N-1} \frac{2n+1}{t-s} \bm I_{s,t}^{(n)} \int_u^v \widetilde{P}_n\left(\frac{r-s}{t-s}\right) \rmd r,
    \end{equation}
    for each increment \([u,v]\subseteq [s,t]\).
\end{restatable}
\begin{restatable}[Properties of $\bW_{u,v}^{(N)}$]{theorem}{bigpoly}
    \label{thm:properties_of_polynomial}
    The polynomial approximation of the Brownian motion in \cref{def:poly_approx_brownian} has the following properties:
        \begin{enumerate}
            \item Converges to Brownian motion in the $\alpha$-H\"older distance, with $\alpha \in [0, \frac 12 )$,\label{item:poly_approx_property1}
            \item The coefficients $\bm I_{s,t}^{(n)}$ are independently and normally distributed with
            \begin{equation}
                \label{eq:polycoeff}
                \bm I_{s,t}^{(n)} \sim \mathcal N\left(\bm 0, \frac{(t-s)}{2n+1}\bm I\right),
            \end{equation}\label{item:poly_approx_property2}
            \item The coefficients $\bm I_{s,t}^{(n)}$ admit closed form Chen relations.\label{item:poly_approx_property3}
        \end{enumerate}
\end{restatable}

\subsection{Training}
Given the polynomial approximation in \cref{def:poly_approx_brownian}, we can write the Strong Stochastic Flow Map as
\begin{equation}
    \bm\Psi_{s,t}^\theta(\bX_s, \bm I_{s,t}^{(N)}) = \bX_s + \bmf_{s,t}^\theta(\bX_s, \bm I_{s,t}^{(N)})(t - s) + \bg_{s,t}^\theta(\bm I_{s,t}^{(N)})(\bW_t - \bW_s),
\end{equation}
where the coefficients up to degree $N$, given by $\bm I_{s,t}^{(N)} = \{\bm I_{s, t}^{(n)}\}_{n\leq N}$, are passed to the neural network terms $\bmf_{s,t}^\theta, \bg_{s,t}^\theta$. Note that $\bW_t - \bW_s = \bm I_{s, t}^{(0)}$.

\subsubsection{Self-distillation objective}
We have seen that the stochastic flow map in \eqref{eq:em-sfm} \emph{is} the It\^o map if and only if the diagonal integrals are equal to the SDE coefficients, $\bm f_{t, t}(x_t) = \bm f(t, x_t), \bm g_{t, t} = \bm g(t)$, and the semigroup condition is satisfied.

These two properties allow us to write down an objective for the It\^o map. This is shown by Theorem \ref{thm:objective}.

\begin{restatable}[Self-distillation Objective]{theorem}{objective}
    \label{thm:objective}
    Let $\bm\Psi_{s, t}(\bX_s, \bW_{[s, t]})$ denote the It\^o map for \eqref{eq:sde}. Then, this map is given by the strong stochastic flow map in \eqref{eq:em-sfm} where $\bm v_{s, t} = [\bm f_{s, t}, \bm g_{s, t}]$ is the unique global minimizer over $\hat{\bm v}$ of
    \begin{equation}
        \label{eq:objective}
        \mathcal{L}_\text{SD}(\hat{\bm v}) = \mathcal{L}_{\bm f, \bm g}(\hat{\bm v}) + \mathcal{L}_\text{D}(\hat{\bm v}),
    \end{equation}
    where 
    \begin{align}
        \mathcal{L}_{\bm f, \bm g}(\hat{\bm v}) &= \mathbb{E}_{t, \bX_t} \left[\lVert \bm f(t, \bX_t) - \hat{\bm f}_{t, t}(\bX_t)\rVert_2^2 + \lVert \bm g(t) - \hat{\bm g}_{t, t}\rVert^2_2 \right],  \\
        \mathcal{L}_\text{D}(\hat{\bm v}) &= \mathbb{E}_{s, u, t, \bX_s, \bW_{[s, t]}} \left[ \lVert \hat{\bm \Psi}_{s, t}(\bX_s, \bW_{[s, t]}) - \hat{\bm \Psi}_{u, t}(\hat{\bm \Psi}_{s, u}(\bX_s, \bW_{[s, u]}), \bW_{[u, t]})\rVert_2^2 \right].
    \end{align}
\end{restatable}

The objective in \eqref{eq:objective} is written in terms of the drift and diffusion coefficients of the SDE for which we wish to learn the It\^o map. In the case of flow models (flow matching, diffusion, stochastic interpolants) the drift coefficient is known conditioned on an end-point sample, $\bX_1$ (and/or $\bX_0$), obtained from the marginal at $t=1$ (and/or $t=0$). 

The results presented in this section are proved in \cref{app:proofs}.

\begin{algorithm}
    \scriptsize
    \caption{Strong Stochastic Flow Map Training}\label{alg:ssfm}
    \begin{algorithmic}[1]
        \Require Batch size $M$, split $\eta \in (0,1)$, polynomial degree $N$,
                 threshold $\Delta t$, EMA decay $\beta \in (0,1)$
        \Repeat

        \State Sample $s \sim \mathcal{U}[0, 1]$, \;
               $\bX_s \sim p_s$
               \hfill $\triangleright$ Simulate or interpolate

        \Statex $\triangleright$ Matching objective \textnormal{(batch size $\lfloor \eta M \rfloor$)}
        \State Sample $t \sim \mathcal{U}[s,\, s {+} \Delta t]$ and
               $\bm{I}_{s,t}^{(N)} \sim$ \eqref{eq:polycoeff}
        \State $\hat\bX_{t} \leftarrow \bX_s
               + \bm{f}(s, \bX_s)(t {-} s)
               + \bm{g}(s)(\bW_{t} {-} \bW_{s})$
        \State $\mathcal{L}_{\bm{f},\bm{g}} \leftarrow
               (t {-} s)^{-1}\,\big\|
               \hat\bX_{t} -
               \bm{\Psi}_{s,t}^{\theta}(\bX_s,\bm{I}_{s,t}^{(N)})
               \big\|^2$

        \Statex $\triangleright$ Distillation objective \textnormal{(batch size $M - \lfloor \eta M \rfloor$)}
        \State Sample $t \sim \mathcal{U}[s {+} \Delta t,\, 1]$ and set
               $u \leftarrow \tfrac{1}{2}(s {+} t)$
        \State Sample $\bm{I}_{s,u}^{(N)}, \bm{I}_{u,t}^{(N)} \sim$ \eqref{eq:polycoeff}
               and compute $\bm{I}_{s,t}^{(N)}$ via \eqref{eq:polychen}
        \State $\bX_{\mathrm{tgt}} \leftarrow
               \texttt{stopgrad}\big(\bm{\Psi}_{s,t}^{\theta}(\bX_s,
               \bm{I}_{s,t}^{(N)})\big)$
        \State $\bX_{\mathrm{pred}} \leftarrow
               \bm{\Psi}_{u,t}^{\theta}\!\big(
               \bm{\Psi}_{s,u}^{\theta}(\bX_s, \bm{I}_{s,u}^{(N)}),\,
               \bm{I}_{u,t}^{(N)}\big)$
        \State $\mathcal{L}_{\mathrm{D}} \leftarrow
               (t {-} s)^{-1}\,\big\|
               \bX_{\mathrm{tgt}} - \bX_{\mathrm{pred}}
               \big\|^2$

        \Statex $\triangleright$ Update
        \State $\theta \leftarrow \theta - \lambda\,\nabla_\theta
               (\mathcal{L}_{\bm{f},\bm{g}} + \mathcal{L}_{\mathrm{D}})$
        \State $\hat\theta \leftarrow \beta\hat\theta + (1{-}\beta)\,\theta$

        \Until{$\theta$ converges}
    \end{algorithmic}
\end{algorithm}

\subsubsection{General training method}

By Theorems \ref{thm:objective} and \ref{thm:properties_of_polynomial}, this parameterization trained according to \eqref{eq:objective} converges to the It\^o map as $N \rightarrow \infty$. In practice, as the variance of the coefficient terms decay with $1/n$, we can obtain a sufficient approximation using a finite number of coefficients.

Note that $\mathcal{L}_{\bm f, \bm g}$ provides a training signal for $s = t$ (\ie $\bm f^\theta_{t, t}, \bm g^\theta_{t, t}$), but $\mathcal{L}_\text{D}$ evaluates the model at $t > s$. This forces the network to generalize from $s = t$ to $t > s$ via continuity. To alleviate this, we instead match a small Euler-Maruyama step (with $t>s$) of the ground truth SDE rather than matching coefficients. This can be shown to result in a weighted coefficient matching objective (see Lemma \ref{appsubsec:em-objective}).

The general training algorithm is given in \cref{alg:ssfm}. This requires a ground truth for $\bm f$ and $\bm g$ which can be constructed based on the task; for diffusion this is obtained by the reverse SDE with $\bm f(t, \bX_t)$ derived from the conditional score -- resulting in a simulation-free objective (see \cref{appsubsec:diffusion-sde}).

\section{Experiments}
\label{sec:experiments}
We consider three experiments that demonstrate the properties and performance of the SSFM model. Firstly, we ablate the algorithmic properties on a non-linear SDE and verify the effectiveness of the polynomial approximation to Brownian motion.
Second, we apply the model to CIFAR-10 and CelebA-64 image generation and show that SSFMs outperform previous deterministic and stochastic flow maps.
Third, we consider generation of equilibrium molecular conformations on the Alanine-Dipeptide dataset, where the SSFM model is capable of generating accurate samples in as few as two network evaluations. 

\subsection{Non-linear SDE}
We first investigate the performance of the SSFM model on a toy system: a non-linear drift, additive-noise SDE, given by
\begin{equation}
    \label{eq:nonlinear-sde}
    \rmd \bX_t = [\bX_t - \bX_t^3]\; \rmd t + \sqrt{\beta_t}\; \rmd \bW_t,
\end{equation}
where $\beta_t$ is a linear interpolation between $\beta_\text{min}=0.1$ and $\beta_\text{max}=20$. This system is intended to mimic the variance preserving reverse diffusion SDE. The SSFM is learned via \cref{alg:ssfm} with the ground truth constructed from \eqref{eq:nonlinear-sde}.

\begin{wrapfigure}{r}{0.5\textwidth}
    \vspace{-3em}
    \centering
    \includegraphics[width=\linewidth]{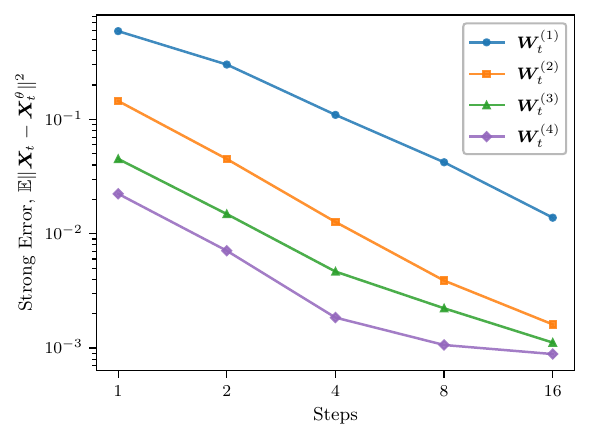}
    \captionof{figure}{Strong error of the SSFM as a function of the polynomial degree.}
    \label{fig:coarse_errors}
    \vspace{-1em}
\end{wrapfigure}

In \cref{fig:coarse_errors}, we ablate the SSFM accuracy as a function of the polynomial degree. As the number of coefficients passed to the model increases, the strong convergence error decreases; this is most pronounced at larger step sizes. It can also been seen that the accuracy gained by each additional coefficient added diminishes, but is far from saturated at $N=4$.

In \cref{fig:poly_brownian}, we show the underlying 4-th degree polynomial expansion of Brownian motion and the learned flow map evaluated at 16-steps. We see that the SSFM is capable of accurately coarsening the SDE dynamics. We include additional plots for this system in \cref{appsubsec:non_linear}.

\begin{figure}
    \centering
    \includegraphics[width=0.85\linewidth]{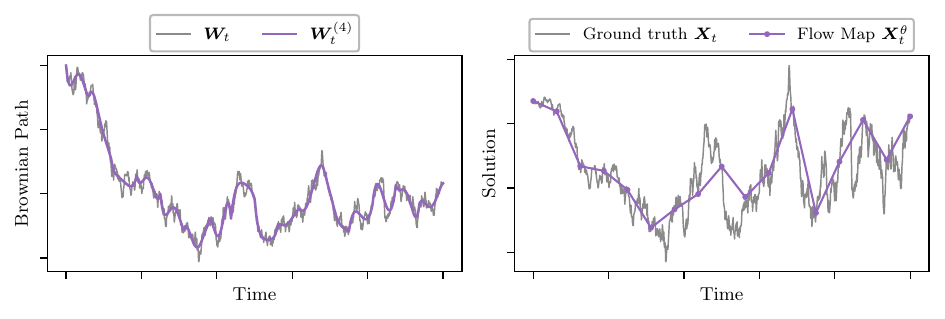}
    \caption{\emph{Left}: Ground-truth Brownian path $\bW_t$ and 4-th degree polynomial approximation ${\bW}_t^{(4)}$ over 16 intervals. \emph{Right}: Ground-truth SDE solution driven by $\bW_t$ and learned 16-step flow map driven by ${\bW}_t^{(4)}$.}
    \label{fig:poly_brownian}
\end{figure}

\subsection{Image generation}
We demonstrate the strong stochastic flow map construction for image generation on the CIFAR-10 and CelebA-64 datasets \parencite{krizhevsky2009learning, liu2015faceattributes}. The training procedure follows \cref{alg:ssfm}, where the ground truth SDE is obtained via the reverse time variance preserving diffusion SDE \parencite{song2021scorebased}. See \cref{appsubsec:diffusion-sde} for a description of this SDE. The $\bm f^\theta, \bm g^\theta$ networks are independently parameterized with the EDM2 architecture \parencite{karras2024analyzing}. The number of polynomial coefficients was chosen to be $N=3$. 

\begin{table}
    \centering
    \small
    \caption{FID ($\downarrow$) on CIFAR-10 and CelebA-64 across different NFE values. We provide deterministic flow maps as a reference, but compare to other stochastic methods.}
    \label{tab:cifar10_celeba64_fid}
    \begin{tabular}{clcccc}
        \toprule
        & & \multicolumn{4}{c}{NFE} \\
        \cmidrule(lr){3-6}
        Dataset & Method & 2 & 4 & 8 & 16 \\
        \midrule

        \multirow{10}{*}{\rotatebox[origin=c]{90}{\centering CIFAR-10}}
        & \textcolor{gray}{\textit{Deterministic flows}} & & & & \\
        & \textcolor{gray}{\hspace{1em} Consistency training \parencite{song2023consistency}} & \textcolor{gray}{5.83} & \textcolor{gray}{---} & \textcolor{gray}{---} & \textcolor{gray}{---} \\
        & \textcolor{gray}{\hspace{1em} Flow map (LSD) \parencite{boffi2025build}}            & \textcolor{gray}{4.37} & \textcolor{gray}{3.34} & \textcolor{gray}{3.33} & \textcolor{gray}{3.57} \\
        & \textcolor{gray}{\hspace{1em} Flow map (PSD-M) \parencite{boffi2025build}}          & \textcolor{gray}{8.43} & \textcolor{gray}{5.96} & \textcolor{gray}{5.07} & \textcolor{gray}{4.64} \\
        & \textcolor{gray}{\hspace{1em} Flow map (PSD-U) \parencite{boffi2025build}}          & \textcolor{gray}{7.95} & \textcolor{gray}{6.03} & \textcolor{gray}{5.32} & \textcolor{gray}{5.16} \\
        & \textcolor{gray}{\hspace{1em} Flow map \parencite{holderrieth2026diamond}}           & \textcolor{gray}{4.60} & \textcolor{gray}{4.18} & \textcolor{gray}{4.88} & \textcolor{gray}{---} \\
        \cmidrule(lr){2-6}
        & \textit{Weak stochastic flows} & & & & \\
        & \hspace{1em} GLASS \parencite{holderrieth2026glass}               & 157.55 & 39.47 & 11.60 & --- \\
        & \hspace{1em} Diamond Map \parencite{holderrieth2026diamond}       & 5.80 & 5.80 & 6.73 & --- \\
        & \textit{Strong stochastic flows} & & & & \\
        & \hspace{1em} SSFM (Ours)                                          & \textbf{4.93} & \textbf{3.49} & \textbf{3.29} & \textbf{3.35} \\

        \midrule

        \multirow{10}{*}{\rotatebox[origin=c]{90}{\centering CelebA-64}}
        & \textcolor{gray}{\textit{Deterministic flows}} & & & & \\
        & \textcolor{gray}{\hspace{1em} Flow map (LSD) \parencite{boffi2025build}}            & \textcolor{gray}{5.74} & \textcolor{gray}{3.18} & \textcolor{gray}{2.18} & \textcolor{gray}{1.96} \\
        & \textcolor{gray}{\hspace{1em} Flow map (PSD-M) \parencite{boffi2025build}}          & \textcolor{gray}{11.75} & \textcolor{gray}{7.89} & \textcolor{gray}{6.06} & \textcolor{gray}{5.09} \\
        & \textcolor{gray}{\hspace{1em} Flow map (PSD-U) \parencite{boffi2025build}}          & \textcolor{gray}{11.02} & \textcolor{gray}{7.47} & \textcolor{gray}{6.00} & \textcolor{gray}{5.63} \\
        & \textcolor{gray}{\hspace{1em} Flow map \parencite{holderrieth2026diamond}}           & \textcolor{gray}{4.05} & \textcolor{gray}{3.08} & \textcolor{gray}{3.15} & \textcolor{gray}{---} \\
        \cmidrule(lr){2-6}
        & \textit{Weak stochastic flows} & & & & \\
        & \hspace{1em} GLASS \parencite{holderrieth2026glass}               & 95.25 & 51.80 & 26.33 & --- \\
        & \hspace{1em} Diamond Map \parencite{holderrieth2026diamond}       & 9.16 & 6.74 & 5.78 & --- \\
        & \textit{Strong stochastic flows} & & & & \\
        & \hspace{1em} SSFM (Ours)                                          & \textbf{5.65} & \textbf{3.89} & \textbf{3.60} & \textbf{3.79} \\

        \bottomrule
    \end{tabular}
\end{table}

\begin{wrapfigure}{r}{0.5\textwidth}
    \vspace{-1em}
    \centering
    \includegraphics[width=\linewidth]{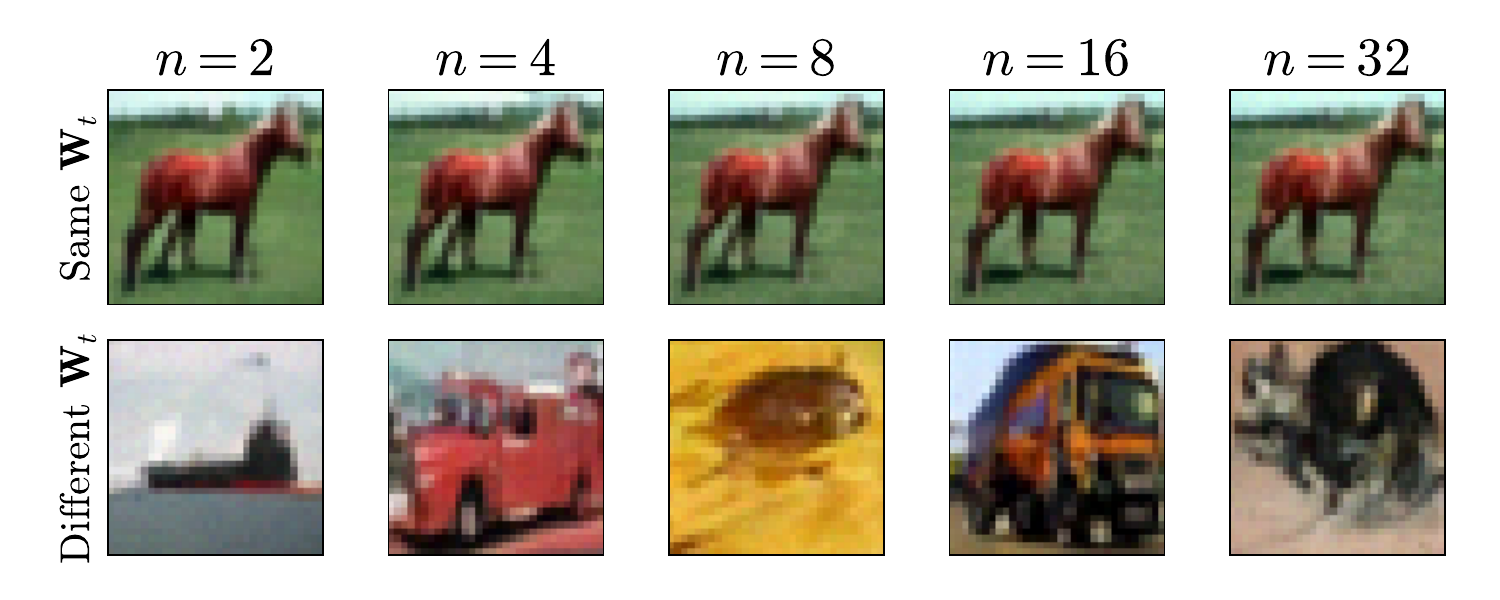}
    \captionof{figure}{\emph{Top}: SSFM with fixed $\bX_0$ and same $\bW_t$ across step counts. \emph{Bottom}: SSFM with fixed $\bX_0$, different $\bW_t$ across step counts.}
    \label{fig:cifar-plot}
\end{wrapfigure}

The results can be seen in \cref{tab:cifar10_celeba64_fid}. The strong stochastic flow map models outperform the weak formulations across all step sizes considered. Additionally, under the strong construction we find that stochastic flow maps become competitive with deterministic flow maps and often outperform.

We verify the strong convergence property of the SSFM in \cref{fig:cifar-plot}. Here, we see that for a fixed initial condition $\bX_0$ and the same Brownian path $\bW_t$, the resulting sample is the same for all step counts. Further, the stochasticity of the flow map is retained; for the same $\bX_0$ with different $\bW_t$, we obtain different generated samples.

\subsection{Molecular systems}
We demonstrate the strong stochastic flow map on the molecular system Alanine Dipeptide (ALDP) \parencite{plainer2026consistent} and all-atom Chignolin. See \cref{appsubsec:molecules} for full details. We compare SSFM to regular diffusion baselines established in \parencite{plainer2026consistent} for ALDP. For all-atom Chignolin, we compare SSFM to a \textit{de novo} denoising diffusion model with comparable parameter count. The results can be seen in \cref{tab:aldp} and \cref{tab:chignolin_w2}. 

\begin{table}
  \centering
  \small
  \caption{PMF squared error ($\downarrow$) and JS divergence ($\downarrow$) on Alanine-Dipeptide across step counts.}\label{tab:aldp}
  \begin{tabular}{lllrrrrrr}
  \toprule
   & & & \multicolumn{6}{c}{NFE} \\
  \cmidrule(lr){4-9}
   & & & 2 & 4 & 10 & 20 & 100 & 1000 \\
  \midrule
  \multirow{6}{*}{PMF Error}
   & Two-for-One & & 16.682          & 15.714          & 11.410          & 3.564          & 0.087          & 0.068          \\
   & Diffusion   & & 16.728          & 15.667          & 11.347          & 3.400          & 0.084          & 0.066          \\
   & Mixture     & & 22.070          & 17.266          & 11.556          & 3.394          & 0.078          & \textbf{0.058} \\
   & Fokker-Planck & & 15.475          & 15.709          & 11.345          & 3.498          & 0.092          & 0.069          \\
   & Both        & & 21.830          & 17.094          & 11.571          & 3.393          & 0.087          & 0.065          \\
   & \textbf{SSFM}        & & \textbf{0.235}  & \textbf{0.168}  & \textbf{0.101}  & \textbf{0.089} & \textbf{0.067} & 0.062          \\
  \midrule
  \multirow{6}{*}{JS ($\times 10^{-2}$)}
   & Two-for-One & & 48.188          & 48.569          & 38.476          & 16.783         & 0.813          & 0.665          \\
   & Diffusion   & & 48.735          & 48.377          & 38.309          & 16.293         & 0.787          & 0.618          \\
   & Mixture     & & 52.432          & 49.500          & 38.709          & 16.279         & 0.770          & 0.609          \\
   & Fokker-Planck & & 47.157          & 48.250          & 38.307          & 16.569         & 0.836          & 0.638          \\
   & Both        & & 52.340          & 49.270          & 38.660          & 16.370         & 0.830          & 0.640          \\
   & \textbf{SSFM}        & & \textbf{1.990}  & \textbf{1.480}  & \textbf{0.950}  & \textbf{0.860} & \textbf{0.640} & \textbf{0.590} \\
  \bottomrule
  \end{tabular}
\end{table}

For ALDP, as expected, the SSFM enables markedly more efficient sampling than the diffusion baselines for step counts $<1000$. As $\text{Steps}\rightarrow 1000$, the methods converge to achieving comparable performance. Notably, the SSFM at $100$ steps is competitive with the diffusion baselines at $1000$ steps, a factor $10$ reduction in NFEs.

For all-atom Chignolin, the SSFM achieves lower error relative to the diffusion baseline across all step counts for the projected backbone torus, and for low step counts in the tICA projection. As expected, the methods converge to achieve comparable performance in the many-step limit. Notably, the SSFM is able to sample accurate configurations in the 1-4 NFE regime. Pooled Ramachandran and tICA plots are shown in \cref{app:experiments}.

\begin{figure}
    \centering
    \includegraphics[width=0.9\linewidth]{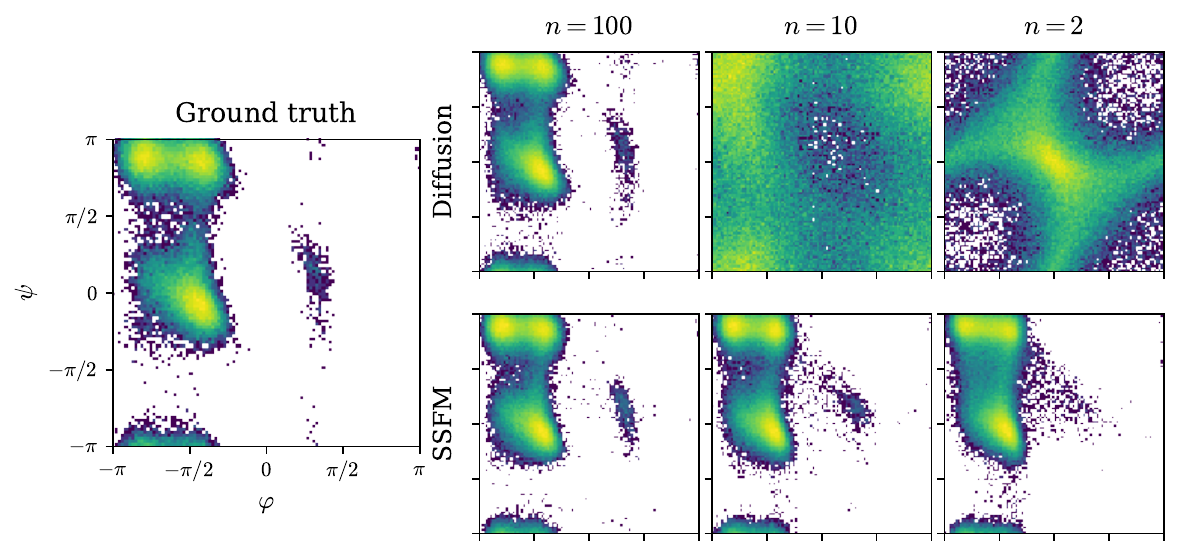}
    \caption{Ramachandran plots for Alanine-Dipeptide showing ground truth data, diffusion mixture baseline and SSFM prediction across step counts.}
    \label{fig:ramas-aldp}
\end{figure}

\begin{table}[t]
  \centering
  \small
  \caption{Chignolin Wasserstein metrics ($\downarrow$) across step counts.}
  \label{tab:chignolin_w2}
  \begin{tabular}{llrrrrrrrr}
  \toprule
   & & \multicolumn{8}{c}{NFE} \\
  \cmidrule(lr){3-10}
   Metric & Method & 1 & 2 & 4 & 8 & 16 & 32 & 64 & 100 \\
  \midrule
   \multirow{2}{*}{$\mathbb{T}$-W$_2$} & Diffusion & 4.737 & 5.882 & 5.587 & 2.716 & 1.793 & 1.714 & 1.699 & 1.735 \\
    & \textbf{SSFM} & \textbf{2.439} & \textbf{2.060} & \textbf{1.824} & \textbf{1.764} & \textbf{1.733} & \textbf{1.707} & \textbf{1.694} & \textbf{1.670} \\
   \midrule
   \multirow{2}{*}{tICA-W$_2$} & Diffusion & 1.976 & 19.147 & 3.236 & 0.712 & \textbf{0.323} & \textbf{0.262} & \textbf{0.266} & 0.330 \\
    & \textbf{SSFM} & \textbf{0.758} & \textbf{0.530} & \textbf{0.414} & \textbf{0.368} & 0.335 & 0.316 & 0.290 & \textbf{0.304} \\
  \bottomrule
  \end{tabular}
\end{table}

\section{Conclusion}
In this work, we have introduced a novel theoretical framework for learning the It\^o map to any additive-noise SDE.
This framework was used to construct a class of flow maps, termed \textsc{Strong Stochastic Flow Maps} (SSFMs), that approximate the It\^o map to compute a strong solution to additive-noise SDEs.
On image generation experiments, this parameterization was shown to outperform previous stochastic flow map models.
When applied to sampling molecular conformations, SSFMs obtained accurate results in as few as 1-2 network evaluations and matched the performance of current diffusion based generative models in the many-step regime. We expect this work to open up further improvements in reward alignment of generative models via pathwise estimators, and to accelerate molecular simulation and generative modeling tasks.

\newpage
\newrefcontext[sorting=nyt]
\printbibliography[heading=bibintoc]

\newpage
\appendix
\crefalias{section}{appendix}
\crefalias{subsection}{appendix}
\crefalias{subsubsection}{appendix}

\startcontents[]
\printcontents[]{l}{1}[3]{{\bfseries \large Appendices}}

\section{Proofs}
\label{app:proofs}
In this section we prove the results provided in the main text. Throughout, we assume that the stochastic differential equations studied satisfy the following assumption.
\begin{assumption}[Unique strong solution]
    \label{ass:unique}
    We assume that the additive-noise stochastic differential equation given by
    \begin{equation*}
        d\bX_t = \bm f(t, \bX_t)\;\rmd t + \bm g(t)\; \rmd \bW_t,
    \end{equation*}
    where $\bm f : \mathbb{R} \times \mathbb{R}^d \rightarrow \mathbb{R}^d, \bm g : \mathbb{R} \rightarrow \mathbb{R}^{d \times w}$ with $\bW_t \in \mathbb{R}^w$ and initial condition $\bX_0$ with finite second moment, has the following properties such that a unique continuous solution $\bX_t$ exists. Let $\bm f, \bm g$ be measurable functions satisfying
    \begin{align*}
        |\bm f(t, \bm x)| + |\bm g(t)| &\leq C(1 + |\bm x|) \\ 
        |\bm f(t, \bm x) - \bm f(t, \bm y)| &\leq D |\bm x - \bm y|
    \end{align*}
    for constants $C, D$. See \textcite{oksendal2003stochastic} for a full discussion on the existence and unique of strong solutions to SDEs.
\end{assumption}

\subsection{Constructing the It\^o map}

\subsubsection{Proof of \texorpdfstring{\cref{lem:tangent}}{Tangent condition}}
\label{proof:tangent}

We first prove a simple lemma that relates the It\^o map to the SDE.

\begin{restatable}[Stochastic Lagrangian equation]{lemma}{lagrangian}
    \label{lem:lagrangian}
    Let $\bm\Psi_{s, t}(\bX_s, \bW_{[s, t]})$ denote the It\^o map for \eqref{eq:sde}. Then,
    \begin{equation}
    \label{eq:lagrangian}
    \rmd \bm \Psi_{s, t}(\bX_s, \bW_{[s, t]}) = \bm f(t, \bm \Psi_{s, t}(\bX_s, \bW_{[s, t]}))\rmd t + \bm g(t)\rmd \bW_t.
    \end{equation}
\end{restatable}

\begin{proof}
    Consider the derivative of the stochastic flow map,
    \begin{align*}
        \rmd \bm \Psi_{s, t}(\bX_s, \bW_{[s, t]}) &= \rmd \bX_t, \\
        &= \bm f(t, \bX_t)\rmd t + \bm g(t)\rmd \bW_t, \\
        &= \bm f(t, \bm \Psi_{s, t}(\bX_s, \bW_{[s, t]}))\rmd t + \bm g(t)\rmd \bW_t.
    \end{align*}
\end{proof}

Now, we are able to prove \cref{lem:tangent} which we restate here.

\begin{theorembox}
    \tangent*
\end{theorembox}

\begin{proof}
    By Lemma \ref{lem:lagrangian} we have that the It\^o map satisfies the stochastic Lagrangian equation. Taking the limit as $s\rightarrow t$, then given the continuity of the It\^o map we have
    \begin{align*}
        \lim_{s \rightarrow t} \rmd \bm \Psi_{s, t}(\bX_s, \bW_{[s, t]}) &= \lim_{s \rightarrow t} \bm f(t, \bm \Psi_{s, t}(\bX_s, \bW_{[s, t]}))\rmd t + \bm g(t)\rmd \bW_t \\
        &= \bm f(t, \bm \Psi_{t, t}(\bX_t, \bW_{[t, t]}))\rmd t + \bm g(t)\rmd \bW_t \\
        &= \bm f(t, \bX_t)\rmd t + \bm g(t)\rmd \bW_t
    \end{align*}
\end{proof}

\subsubsection{Proof of \texorpdfstring{\cref{prop:em-sfm}}{Strong stochastic flow map}}
\label{proof:em-sfm}

Next, we prove \cref{prop:em-sfm}.
\begin{theorembox}
    \ssfm*
\end{theorembox}

\begin{proof}
    Consider the application of It\^o's lemma to $\bm \Psi_{s, t}(\bX_s, \bW_{[s, t]})$, given by
    \begin{align*}
        \rmd \bm \Psi_{s, t}(\bX_s, \bW_{[s, t]}) &= \left(\partial_t \bm \Psi_{s, t}(\bX_s, \bW_{[s, t]}) + \frac{1}{2} \partial^2_{\bW_t} \bm \Psi_{s, t}(\bX_s, \bW_{[s, t]})\right)\rmd t + \partial_{\bW_t}\bm \Psi_{s, t}(\bX_s, \bW_{[s, t]}) \rmd \bW_t, \\
        &= \bigg(\bm f_{s, t}(\bX_s, \bW_{[s, t]}) + (t-s)\partial_t \bm f_{s, t}(\bX_s, \bW_{[s, t]}) + (\bW_t - \bW_s)\partial_t \bm g_{s, t}(\bW_{[s, t]}) \\ 
        &+ \frac{1}{2}(t-s)\partial^2_{\bW_t}\bm f_{s, t}(\bX_s, \bW_{[s, t]}) + \frac{1}{2}(\bW_t - \bW_s)\partial^2_{\bW_t}\bm g_{s, t}(\bW_{[s, t]}) + \partial_{\bW_t}\bm g_{s, t}(\bW_{[s, t]})\bigg)\rmd t \\
        &+ \Big( (t-s)\partial_{\bW_t}\bm f_{s, t}(\bX_s, \bW_{[s, t]}) + \bm g_{s, t}(\bW_{[s, t]}) + (\bW_t - \bW_s)\partial_{\bW_t}\bm g_{s, t}(\bW_{[s, t]})\Big)\rmd \bW_t.
    \end{align*}
    Now, taking the limit $s\rightarrow t$, we obtain
    \begin{equation*}
        \lim_{s\rightarrow t} \rmd \bm \Psi_{s, t}(\bX_s, \bW_{[s, t]}) = \Big(\bm f_{t, t}(\bX_t, \bW_{[t, t]}) + \partial_{\bW_t}\bm g_{t, t}(\bW_{[t, t]})\Big) \rmd t + \bm g_{t, t}\rmd \bW_t.
    \end{equation*}
    Finally, using the conditions that $\bm f_{t, t}(\bX_t, \bW_{[t, t]}) = \bm f_{t, t}(\bX_t)$ and $\bm g_{t, t}(\bW_{[t, t]}) = \bm g_{t, t}$, this reduces to
    \begin{equation*}
        \lim_{s\rightarrow t} \rmd \bm \Psi_{s, t}(\bX_s, \bW_{[s, t]}) = \bm f_{t, t}(\bX_t)\rmd t + \bm g_{t, t}\rmd \bW_t.
    \end{equation*}
    To satisfy the tangent condition \eqref{eq:tangent} we must have $\bm f_{t, t}(\bX_t) = \bm f(t, \bX_t)$ and $\bm g_{t, t} = \bm g(t)$. Conversely, substituting $\bm f_{t, t}(\bX_t) = \bm f(t, \bX_t), \bm g_{t, t} = \bm g(t)$ recovers the tangent condition.
\end{proof}

\subsubsection{Proof of \texorpdfstring{\cref{prop:semigroup}}{Semigroup condition}}
\label{proof:semigroup}

Now onto the proof that the strong stochastic flow map attains the It\^o map when satisfying the tangent and semigroup conditions.
\begin{theorembox}
    \semigroup*
\end{theorembox}

\begin{proof}
    First observe that the It\^o map satisfies the semigroup condition,
    \begin{align*}
        \bm \Psi_{s, t}(\bX_s, \bW_{[s, t]}) &= \bm \Psi_{u, t}(\bm \Psi_{s, u}(\bX_s, \bW_{[s, u]}), \bW_{[u, t]}) \\
        &= \bm \Psi_{u, t}(\bX_u, \bW_{[u, t]}) \\
        &= \bX_t
    \end{align*}
    Next, we want to show the inverse implication. Let $\bm \Psi_{s, t}(\bX_s, \bW_{[s, t]})$ denote the strong stochastic flow map satisfying \eqref{eq:tangent}, \eqref{eq:em-sfm} and \eqref{eq:semigroup}. We show that this map is the It\^o map for \eqref{eq:sde}. This follows by considering the semigroup condition,
    \begin{align*}
        \bm \Psi_{s, t+h}(\bX_s, \bW_{[s, t+h]}) &= \bm \Psi_{t, t+h}(\bm \Psi_{s, t}(\bX_s, \bW_{[s, t]}), \bW_{[t, t+h]}) \\
        &= \bm \Psi_{s, t}(\bX_s, \bW_{[s, t]}) + h\bm f_{t, t+h}(\bm \Psi_{s, t}(\bX_s, \bW_{[s, t]}), \bW_{[t, t+h]}) \\
        &+ (\bW_{t+h} - \bW_t)\bm g_{t, t+h}(\bW_{[t, t+h]}) \\
    \end{align*}
    Define $\hat{\bX}_t = \bm \Psi_{s, t}(\bX_s, \bW_{[s, t]})$ and re-arrange to get
    \begin{equation*}
        \hat{\bX}_{t+h} - \hat{\bX}_t = h\bm f_{t, t+h}(\hat{\bX}_t, \bW_{[t, t+h]}) + (\bW_{t+h} - \bW_t)\bm g_{t, t+h}(\bW_{[t, t+h]}).
    \end{equation*}
    Next consider taking the following sum over a partition $t = u_0 < u_1 < \dots < u_n = t+h$, 
    \begin{equation*}
        \sum_{i} \hat{\bX}_{u_{i+1}} - \hat{\bX}_{u_i} = \sum_{i} (u_{i+1} - u_i) \bm f_{u_i, u_{i+1}}(\hat{\bX}_{u_i}, \bW_{[u_i, u_{i+1}]}) + \sum_{i} (\bW_{u_{i+1}} - \bW_{u_i}) \bm g_{u_i, u_{i+1}}(\bW_{[u_i, u_{i+1}]})
    \end{equation*}
    The LHS of this sum telescopes to $\hat{\bX}_{t+h} - \hat{\bX}_t$. We now take the limit of this partition with $\max_i (u_{i+1} -u_i)\rightarrow 0$. 
    
    Since we have $\bm f_{u, u}(\hat{\bX}_u, \bW_{[u, u]}) = \bm f_{u, u}(\hat{\bX}_u)$ and $\bm g_{u, u}(\bW_{[u, u]}) = \bm g_{u, u}$ by the independence assumption on \eqref{eq:em-sfm}, consider the following equivalent expression obtained by adding and subtracting the diagonal coefficients,
    \begin{align*}
        \hat{\bX}_{t+h} - \hat{\bX}_t &= \sum_{i} (u_{i+1} - u_i) \bm f_{u_i, u_i}(\hat{\bX}_{u_i}) + \sum_{i} (u_{i+1} - u_i) \Big( \bm f_{u_i, u_{i+1}}(\hat{\bX}_{u_i}, \bW_{[u_i, u_{i+1}]}) - \bm f_{u_i, u_i}(\hat{\bX}_{u_i})\Big) \\
        &+ \sum_{i} (\bW_{u_{i+1}} - \bW_{u_i}) \bm g_{u_i, u_i} + \sum_{i} (\bW_{u_{i+1}} - \bW_{u_i}) \Big( \bm g_{u_i, u_{i+1}}(\bW_{[u_i, u_{i+1}]}) - \bm g_{u_i, u_i}\Big)
    \end{align*}
    The two difference terms (in both $\bm f$ and $\bm g$) decay to zero as $\max_i (u_{i+1} -u_i)\rightarrow 0$. This follows from the Lipschitz-in-time assumption on $\bm f, \bm g$. Specifically, for the $\bm f$ residual we have
    \begin{align*}
        \Big|\sum_{i} (u_{i+1} - u_i) \Big( \bm f_{u_i, u_{i+1}}(\hat{\bX}_{u_i}, \bW_{[u_i, u_{i+1}]}) - \bm f_{u_i, u_i}(\hat{\bX}_{u_i})\Big) \Big| &\leq  L_f \sum_{i} |u_{i+1} - u_i||u_{i+1} - u_i|, \\
        &\leq L_f \max_i (|u_{i+1} - u_i|) \sum_i |u_{i+1} - u_i|, \\
        &\leq L_f h \max_i (|u_{i+1} - u_i|),
    \end{align*}
    which tends to zero as $\max_i (u_{i+1} -u_i)\rightarrow 0$. For the $\bm g$ residual, we have
    \begin{align*}
        \Big| \sum_{i} (\bW_{u_{i+1}} - \bW_{u_i}) \Big( \bm g_{u_i, u_{i+1}}(\bW_{[u_i, u_{i+1}]}) - \bm g_{u_i, u_i}\Big) \Big| &\leq L_g \sum_i | (\bW_{u_{i+1}} - \bW_{u_i})| |u_{i+1} - u_i| \\
        &\leq L_g \max_i (|\bW_{u_{i+1}} - \bW_{u_i}|) \sum_i |u_{i+1} - u_i|, \\
        &= hL_g \max_i (|\bW_{u_{i+1}} - \bW_{u_i}|),
    \end{align*}
    which tends to zero as $\max_i (u_{i+1} -u_i)\rightarrow 0$ by the uniform continuity of $\bW_t$. We are left with
    \begin{equation*}
        \hat{\bX}_{t+h} - \hat{\bX}_t = \lim_{\max_i (u_{i+1} -u_i)\rightarrow 0}\sum_{i} (u_{i+1} - u_i) \bm f_{u_i, u_i}(\hat{\bX}_{u_i}) + \sum_{i} (\bW_{u_{i+1}} - \bW_{u_i}) \bm g_{u_i, u_i}.
    \end{equation*}
    
    Therefore, both Riemann and It\^o (Stratonovich, equivalently) converge \parencite{oksendal2003stochastic} to give
    \begin{align*}
        \hat{\bX}_{t+h} &= \hat{\bX}_t + \int_t^{t+h}\bm f_{u, u}(\hat{\bX}_u)\rmd u + \int_t^{t+h}\bm g_{u, u}\rmd \bW_u \\
        &= \hat{\bX}_t + \int_t^{t+h}\bm f(u, \hat{\bX}_u)\rmd u + \int_t^{t+h}\bm g(u)\rmd \bW_u,
    \end{align*}
    where $\bm f_{t, t}(\bX_t) = \bm f(t, \bX_t)$ and $\bm g_{t, t} = \bm g(t)$ follows from Proposition \ref{prop:em-sfm}. Since this holds for any $t\geq s$ and $h\geq 0$, $\hat{\bX}_t$ is a strong solution to the SDE with coefficients $\bm f(t, \bX_t), \bm g(t)$. By uniqueness, this implies that $\bm \Psi_{s, t}(\bX_s, \bW_{[s, t]})$ is the It\^o map for \eqref{eq:sde}.
\end{proof}

\subsubsection{Proof of \cref{thm:objective}}
Given the results above, we are able to prove \cref{thm:objective}.

\begin{theorembox}
    \objective*
\end{theorembox}

\begin{proof}
    We have that for any $\hat{\bm v}$,
    \begin{align*}
        \mathcal{L}_{\bm f, \bm g}(\hat{\bm v}_{t, t}) &\geq \mathcal{L}_{\bm f, \bm g}(\bm v_{t, t}) = 0, \\
        \mathcal{L}_\text{D}(\hat{\bm v}) &\geq 0.
    \end{align*}
    This follows since $\mathcal{L}_{\bm f, \bm g}$ is convex in $\hat{\bm v}_{t, t}$ with unique global minimizer $\bm v_{t, t}$. Note that the It\^o map satisfies $\mathcal{L}_\text{D}(\bm v) = 0$. This implies that for the It\^o map we obtain the minimum of the objective, given by
    \begin{equation*}
        \mathcal{L}_\text{SD}(\bm v) = 0.
    \end{equation*}
    To show that the minimizer is unique, consider any $\hat{\bm v}$ (and associated $\hat{\bm \Psi}$) that obtains the minimum,
    \begin{equation*}
        \mathcal{L}_{SD}(\hat{\bm v}) = 0
    \end{equation*}
    Then we must have
    \begin{align*}
        \hat{\bm v}_{t, t} &= \bm v_{t, t} \\
        \mathcal{L}_\text{D}(\hat{\bm v}) &= 0.
    \end{align*}
    By Propositions \ref{prop:em-sfm} (since $\hat{\bm v}_{t, t} = \bm v_{t, t}$, the tangent condition holds) and \ref{prop:semigroup} (since $\mathcal{L}_\text{D}=0$, the semigroup condition holds), this implies $\hat{\bm \Psi}$ is the It\^o map for \eqref{eq:sde}.
\end{proof}

\subsection{Constructing the Brownian motion}
\subsubsection{Primer on rough path theory}
A (step-2) rough path is an element of \(G^2(\mathbb{R}^d) = \mathbb{R}^d\oplus\mathbb{R}^{d\times d}\). We define a metric on \(\mathbb{X}=(\mathbb{X}^{(1)},\mathbb{X}^{(2)})\in G^2(\mathbb{R}^d)\) by,
\[
    d(\mathbb{X}_s,\mathbb{X}_t)\coloneq |\mathbb{X}_{s,t}^{(1)}|+|\mathbb{X}^{(2)}_{s,t}|^\frac 12,
\]
where we have used \(|\cdot |\) to indicate the \textit{Euclidean norm on \(\mathbb{R}^{\otimes k}\)} for the appropriate choice of \(k\in\{1,2\}\). To measure the ``difference'' between two rough paths \(\mathbb{X},\mathbb{Y}\in G^2(\mathbb{R}^2)\) we define a pathwise metric, for each \(\alpha\in[0,\frac 12)\), the \(\alpha\)-H\"older metric by,
\[
    d_{\alpha\text{-H\"ol}}(\mathbb{X},\mathbb{Y}) \coloneq \sup_{0\leq s<t\leq 1}\Big (\frac{|\mathbb{X}_{s,t}^{(1)}-\mathbb{Y}_{s,t}^{(1)}| + |\mathbb{X}_{s,t}^{(2)}-\mathbb{Y}_{s,t}^{(2)}|^{\frac 12}}{|t-s|^\alpha}\Big ).
\]
To avoid ambiguity of indices being used to represent both the \textit{level} of a rough path, and in  \cref{def:poly_approx_brownian} the order of our approximation, we will sometimes use the projection operator, \(\pi_k:G^2(\mathbb{R}^d)\rightarrow (\mathbb{R}^d)^{\otimes k}\) for \(k=1,2\), to explicitly distinguish the first and second levels of the rough path respectively, i.e. \(\pi_k(\mathbb{X})=\mathbb{X}^{(k)}\). 

For both Brownian motion \(\boldsymbol{W}_t\) and the polynomial approximation \(\boldsymbol{W}_t^{(N)}\) from \cref{def:poly_approx_brownian}, we lift the paths to rough paths \(\mathbb{W}\) and \(\mathbb{W}^{(N)}\) respectively by defining the canonical Stratonovich lifts,
\[
\begin{aligned}
    \pi_2(\mathbb{W}_{s,t})&:=\int_s^t \boldsymbol{W}_{s,u}\otimes\,\text{d}\boldsymbol{W}_u,\\
    \pi_2(\mathbb{W}_{s,t}^{(N)})&:=\int_s^t \boldsymbol{W}_{s,u}^{(N)}\otimes\,\text{d}\boldsymbol{W}_u^{(N)}.
\end{aligned}
\]

\subsubsection{Proof of \cref{thm:properties_of_polynomial}}
\begin{theorembox}
    \bigpoly*
\end{theorembox}
\begin{proof}
    Property \ref{item:poly_approx_property2} is simply an application of It\^o's isometry, using the fact that the \(n^\text{th}\) (shifted) Legendre polynomial has a normalisation constant of \(\frac{1}{2n+1}\),
    \[
        \boldsymbol{I}_{s,t}^{(n)}=\int_s^t\widetilde{P}_n\Big (\frac{r-s}{t-s}\Big )\,\text{d}\boldsymbol{W}_r\sim \mathcal{N}\big (0,K(s,t)\big ),
    \]
    where,
    \[
    \begin{aligned}
        K(s,t) &= \mathbb{E}\Big [\int_s^t\widetilde{P}_n\Big (\frac{r-s}{t-s}\Big )\,\text{d}\boldsymbol{W}_r\otimes \int_s^t\widetilde{P}_n\Big (\frac{r-s}{t-s}\Big )\,\text{d}\boldsymbol{W}_r\Big ]\\
        &=\boldsymbol{I}\int_s^t\widetilde{P}_n\Big (\frac{r-s}{t-s}\Big )^2\,\text{d}r\\
        &=\frac{t-s}{2n+1}\boldsymbol{I}.
    \end{aligned}
    \]
    Property \ref{item:poly_approx_property1} and \ref{item:poly_approx_property3} require slightly more work, so are proven individually as \cref{prop:polynomial-convergence} and \cref{prop:chen} respectively.
\end{proof}
The following lemma will be required to establish rough-path convergence of the polynomial approximation \(\boldsymbol{W}_t\).
\begin{restatable}[Uniform bounds on \(\mathbb{W}\)]{lemma}{unif_bounds_W}\label{lem:unif_bounds_W}
    For \(\alpha\in[0,\frac 12)\), there exists \(C_\alpha\in L^2(\mathbb{P})\) such that for all \([u,v]\subseteq[s,t]\),
    \begin{equation}\label{eq:unif_bounds_w}
        \|\mathbb{W}_{u,v}\|\leq C_\alpha|v-u|^{\alpha}
    \end{equation}
    where \(\|\cdot \|:G^2(\mathbb{R}^d)\rightarrow \mathbb{R}\) denotes the norm,
    \[
        \|\mathbb{X}_{u,v}\|=\max\Big (\big \|\mathbb{X}_{u,v}^{(1)}\big \|_{L^2}, \big \|\mathbb{X}_{u,v}^{(2)}\big \|_{L^2}^{\frac 12}\Big ).
    \]
\end{restatable}
\begin{proof}
    This is a standard application of the Garsia–Rodemich–Rumsey lemma \parencite[Lemma 1.1]{garsia1970}, which can be generalised to any metric as seen in \parencite[Proposition A.8]{friz2010multidimensional}. Choosing \(p(u)=|u|^{1/2}\), \(\Psi(u)=u^q\) for some \(q\geq \frac{4}{1-2\alpha}\) and \(B=\int_0^1\int_0^1\Psi\big (\frac{\|\mathbb{W}_{u,v}\|}{p(v-u)}\big )\,\text{d}u\,\text{d}v\), we obtain,
    \[
        \|\mathbb{W}_{u,v}\|\leq C_\alpha|v-u|^{1/2 - 2/q}\leq C_\alpha|v-u|^{\alpha},
    \]
    where \(C_\alpha:=\frac{8q(4B)^{1/q}}{q-4}\). To verify \(C_\alpha\in L^2(\mathbb{P})\) we use \(\|\mathbb{W}_{u,v}\|_{L^2(\mathbb{P})}=|v-u|^{1/2}\)
    \[
    \begin{aligned}
        \mathbb{E}[C_\alpha^2]&\lesssim\mathbb{E}[B^{2/q}]\\
        &\leq \int_0^1\int_0^1\frac{\|\mathbb{W}_{u,v}\|^2_{L^2(\mathbb{P})}}{|v-u|^{2\alpha}}\,\text{d}u\,\text{d}v\\
        &=\int_0^1\int_0^1|v-u|^{1-2\alpha}\,\text{d}u\,\text{d}v\\
        &<\infty.
    \end{aligned}
    \]
\end{proof}
\begin{restatable}[Rough path convergence of the polynomial approximation]{proposition}{poly_convergence}
    \label{prop:polynomial-convergence}
    The polynomial expansion of Brownian motion given in \cref{def:poly_approx_brownian} converges in the rough-path sense, i.e. for any \(\alpha\in[0,\frac 12)\),
    \[
        d_{\alpha-\text{H\"ol}}(\mathbb{W},\mathbb{W}^N)\overset{\text{a.s.}}\longrightarrow 0,
    \]
    as \(N\rightarrow \infty\).
\end{restatable}
\begin{proof}
    As shown in \parencite[Theorem 2.4]{foster2020optimal}, we can define the following filtration,
    \[
        \big \{\mathscr{F}_N:=\sigma(\{\boldsymbol{I}^{(n)}_{s,t}\ :\ n=0,1,\cdots ,N\})\big \}_{N\geq 0},
    \]
    so that the polynomial approximation admits the representation,
    \[
        \mathbb{W}_{u,v}^{(N)}=\mathbb{E}[\mathbb{W}_{u,v}\,|\, \mathscr{F}_N],
    \]
    for any \([u,v]\subseteq [s,t]\). Thus, by taking expectations with respect to \(\mathscr{F}_N\) in \cref{lem:unif_bounds_W},
    \begin{equation}\label{eq:unif_bounds_wN}
        \|\mathbb{W}_{u,v}^{(N)}\|\leq \widetilde{C}_\alpha|v-u|^{\alpha},
    \end{equation}
    where \(\widetilde{C}_\alpha\coloneq \sup_{n\geq 0}\mathbb{E}[C_\alpha\,|\,\mathscr{F}_n]\) and where we continue to use \(\|\cdot \|\) as defined in \cref{lem:unif_bounds_W}. By Doob's maximal inequality \(\widetilde{C}_\alpha\) is finite a.s. since \(\mathbb{E}[C_\alpha\,|\,\mathscr{F}_n]\) is a (discrete) martingale.

    Combining \cref{eq:unif_bounds_w} and \cref{eq:unif_bounds_wN}, the sequence \(\{\mathbb{W}_{u,v}-\mathbb{W}_{u,v}^{(N)}\}_{N\geq 0}\) is uniformly bounded and uniformly equicontinuous. Thus, invoking the Arzel\`a-Ascoli theorem, we conclude that there exists a uniformly convergent subsequence and since pointwise convergence is proven in \parencite{foster2020optimal}, we know that this limit is zero, so,
    \[
        \mathbb{W}_{u,v}^{(N)}\overset{a.s.}\longrightarrow \mathbb{W}_{u,v}\qquad \text{uniformly as }N\rightarrow \infty.
    \]
        
    Finally, defining \(C^* = \max(C_\beta,\widetilde{C}_\beta)^{\frac \alpha \beta}\), use the following inequality,
    \[
    \begin{aligned}
        \frac{\|\pi_k(\mathbb{W}_{u,v}-\mathbb{W}_{u,v}^{(N)})\|_{L_2}}{|v-u|^{k\alpha}}&\leq \bigg (\frac{\|\pi_k(\mathbb{W}_{u,v}-\mathbb{W}_{u,v}^{(N)})\|_{L_2}}{|v-u|^\beta}\bigg )^{\frac \alpha \beta}\Big  (\,\underset{0\leq s<t\leq 1}{\text{sup}}\|\pi_k(\mathbb{W}_{u,v}-\mathbb{W}_{u,v}^{(N)})\|_{L_2}\Big  )^{1-\frac \alpha \beta}\\
        &\leq C^*\Big  (\,\underset{0\leq u<v\leq 1}{\text{sup}}\|\pi_k(\mathbb{W}_{u,v}-\mathbb{W}_{u,v}^{(N)})\|_{L_2}\Big  )^{1-\frac \alpha \beta}\\
        &\overset{a.s.}\longrightarrow 0\qquad \text{uniformly as }N\rightarrow \infty.
    \end{aligned}
    \]
    Thus, \(d_{\alpha-\text{H\"ol}}(\mathbb{W},\mathbb{W}^{(N)})\overset{a.s.}\longrightarrow 0\) as \(N\rightarrow \infty\).
\end{proof}

Finally, we move to deriving the explicit Chen relations for the integrals \(\boldsymbol{I}_{s,t}^{(n)} \) To achieve this, we first establish the following related dilation rules for the (shifted) Legendre polynomials.
\begin{restatable}[Dilation rule for shifted Legendre polynomials]{lemma}{legendre_dilation}\label{lem:legendre_dilation}
    Let \(\widetilde{P}_n(x)\) denote the \(n^\text{th}\) shifted Legendre polynomial on \([0,1]\). Then,
    \begin{equation}
        \widetilde{P}_n(x) = \sum_{m=0}^n c_{n,m}\,\widetilde{P}_m(2x)=\sum_{m=0}^n(-1)^{n+m}c_{n,m}\widetilde{P}_m(2x-1),
    \end{equation}
    where,
    \begin{equation*}
        c_{n,m} :=(-1)^{n}(2m+1)\sum_{k=m}^n\Big (-\frac 12\Big )^k \frac{(n+k)!}{(n-k)!(k-m)!(k+m+1)!}.
    \end{equation*}
\end{restatable}
\begin{proof}
    Using the orthogonality relations for Legendre polynomials, we derive the following integral expression for \(c_{n,m}\),
    \[
        (2m+1)\int_0^1\widetilde{P}_n\Big (\frac u2\Big )\widetilde{P}_m(u)\,\text{d}u=(2m+1)\sum_{m=0}^nc_{n,m}\int_0^1\widetilde{P}_n(u)\widetilde{P}_k(u)\,\text{d}u=c_{n,m}.
    \]
    We can evaluate the integral explicitly. By substituting Rodrigues' formula for \(\widetilde{P}_m(u)\), before applying integration by parts \(m\)-times (noting that the boundary terms always vanish),
    \[
    \begin{aligned}
        \int_0^1\widetilde{P}_n\Big (\frac u2\Big )P_m(u)\,\text{d}u&=\frac{1}{m!}\int_0^1\widetilde{P}_n\Big (\frac u2\Big )\frac{\text{d}^m}{\text{d}u^m}[x^m(x-1)^m]\,\text{d}u\\
        &=\frac{1}{m!}\int_0^1x^m(1-x)^m\frac{\text{d}^m}{\text{d}u^m}\widetilde{P}_n\Big (\frac u2\Big )\,\text{d}u.
    \end{aligned}
    \]
    Next, substituting the explicit form for the shifted Legendre polynomial,
    \[
        \widetilde{P}_n\Big (\frac u2\Big )=(-1)^n\sum_{k=0}^n\binom{n}{k}\binom{n+k}{k}\Big (-\frac x2\Big )^k,
    \]
    we obtain,
    \[
    \begin{aligned}
        \int_0^1\widetilde{P}_n\Big (\frac u2\Big )P_m(u)\,\text{d}u&=\frac{1}{m!}\int_0^1x^m(1-x)^m\frac{\text{d}^m}{\text{d}u^m}\Big [(-1)^n\sum_{k=0}^n\binom{n}{k}\binom{n+k}{k}\Big (-\frac x2\Big )^k\,\Big ]\,\text{d}u\\
        &=\frac{(-1)^{n}}{m!}\sum_{k=m}^n\Big (-\frac 12\Big )^k\binom{n}{k}\binom{n+k}{k}\int_0^1x^m(1-x)^m\frac{\text{d}^m}{\text{d}u^m} [x^k\, ]\,\text{d}u\\
        &=\frac{(-1)^{n}}{m!}\sum_{k=m}^n\Big (-\frac 12\Big )^k \frac{k!}{(k-m)!}\binom{n}{k}\binom{n+k}{k}\int_0^1x^k(1-x)^m\,\text{d}u\\
        &=\frac{(-1)^{n}}{m!}\sum_{k=m}^n\Big (-\frac 12\Big )^k \frac{k!}{(k-m)!}\binom{n}{k}\binom{n+k}{k}\frac{k!\,m!}{(k+m+1)!}\\
        &=(-1)^{n}\sum_{k=m}^n\Big (-\frac 12\Big )^k \frac{(n+k)!}{(n-k)!(k-m)!(k+m+1)!}.
    \end{aligned}
    \]
    For the second equation, use the symmetry of the Legendre polynomial,
    \[
    \begin{aligned}
        \widetilde{P}_n(x)&=(-1)^n\widetilde{P}_n(1-x)\\
        &=(-1)^n\sum_{m=0}^nc_{n,m}\widetilde{P}_m\big (2(1-x)\big )\\
        &=(-1)^n\sum_{m=0}^nc_{n,m}\widetilde{P}_m\big (1-(2x-1)\big )\\
        &=\sum_{m=0}^n(-1)^{n+m}c_{n,m}\widetilde{P}_m (2x-1).
    \end{aligned}
    \]
\end{proof}
The proof of \cref{prop:chen} now follows naturally, by splitting the domain of the integral \(\boldsymbol{I}_{s,t}^{(n)}\) into two and applying the relations we have just derived.
\begin{restatable}[Chen relations for \(\boldsymbol{I}^{(n)}\)]{proposition}{chen_relations}
    \label{prop:chen}
    Let \(u=\dfrac{s+t}{2}\) be the midpoint of times \(s\leq t\). Then,
    \begin{equation}\label{eq:polychen}
        \boldsymbol{I}_{s,t}^{(n)}=\sum_{m=0}^nc_{n,m}\Big (\boldsymbol{I}_{s,u}^{(m)}+(-1)^{n+m}\boldsymbol{I}_{u,t}^{(m)}\Big ),
    \end{equation}
    where \(c_{n,m}\) are defined as in \cref{lem:legendre_dilation}.
\end{restatable}
\begin{proof}
    By splitting \([s,t]\) into the two half-domains \([s,u]\) and \([u,t]\), we can apply the dilation rules for the Legendre polynomials given in \cref{lem:legendre_dilation}. We use the identities,
    \[
    \begin{aligned}
        2\Big (\frac{r-s}{t-s}\Big ) &= \frac{r-s}{u-s},\\[6pt]
        2\Big (\frac{r-s}{t-s}\Big ) - 1 &=\frac{r-u}{t-u},
    \end{aligned}
    \]
    so that,
    \[
    \begin{aligned}
        \boldsymbol{I}_{s,t}^{(n)}&=\int_s^t \widetilde{P}_n\Big (\frac{r-s}{t-s}\Big )\,\text{d}\boldsymbol{W}_r\\
        &=\sum_{m=0}^nc_{n,m}\bigg (\int_s^{u}\widetilde{P}_m\Big (\frac{r-s}{u-s}\Big )\,\text{d}\boldsymbol{W}_r+(-1)^{n+m}\int_{u}^t\widetilde{P}_m\Big (\frac{r-u}{t-u}\Big )\,\text{d}\boldsymbol{W}_r\bigg )\\
        &=\sum_{m=0}^nc_{n,m}\Big (\boldsymbol{I}_{s,u}^{(m)}+(-1)^{n+m}\boldsymbol{I}_{u,t}^{(m)}\Big ),
    \end{aligned}
    \]
    as required.
\end{proof}

\section{Training}
\label{app:training}

\subsection{Euler-Maruyama step objective}
\label{appsubsec:em-objective}
Here we show that matching an Euler-Maruyama step with the strong stochastic flow map model is equivalent to a weighted coefficient loss.

\begin{restatable}[Euler-Maruyama objective]{lemma}{emobj}
    \label{lem:emobj}
    Consider the following objective
    \begin{equation*}
        \mathcal{L}_\text{EM} = \big\|\bX_t - \hat{\bX}_t\big\|^2,
    \end{equation*}
    where $\hat{\bX}_t = \bm\Psi_{s,t}(\bX_s, \bW_{[s, t]})$ via \eqref{eq:em-sfm} and
    \begin{equation*}
        \bX_t = \bX_s + \bm f(s, \bX_s) (t-s) + \bm g(s)(\bW_t - \bW_s).
    \end{equation*}
    Since the true coefficients $\bm f(s, \bX_s), \bm g(s)$ depend only on $s$, it suffices to restrict the model to coefficients $\bm f_{s,t}(\bX_s)$ and $\bm g_{s,t}$ that are independent of $\bW_{[s, t]}$. Then
    \begin{equation*}
        \mathbb{E}[\mathcal{L}_\text{EM} \mid \bX_s] = (t-s)^2 \mathcal{L}_f + (t-s) \mathcal{L}_g,
    \end{equation*}
    where $\mathcal{L}_f = \left(\bm f(s, \bX_s) - \bm f_{s,t}(\bX_s)\right)^2$ and $\mathcal{L}_g = \left(\bm g(s) - \bm g_{s,t}\right)^2$.
\end{restatable}

\begin{proof}
    Under the independence assumption, $\bm f_{s,t}(\bX_s)$ and $\bm g_{s,t}$ are independent of $\bW_{[s, t]}$, so
    \begin{equation*}
        \bX_t - \hat{\bX}_t = \left(\bm f(s, \bX_s) - \bm f_{s,t}(\bX_s)\right)(t-s) + \left(\bm g(s) - \bm g_{s,t}\right)(\bW_t - \bW_s).
    \end{equation*}
    Expanding the square,
    \begin{align*}
        \mathcal{L}_\text{EM} &= (t-s)^2 \left(\bm f(s, \bX_s) - \bm f_{s,t}(\bX_s)\right)^2 \\
        &\quad + 2(t-s)\left(\bm f(s, \bX_s) - \bm f_{s,t}(\bX_s)\right)\left(\bm g(s) - \bm g_{s,t}\right)(\bW_t - \bW_s) \\
        &\quad + \left(\bm g(s) - \bm g_{s,t}\right)^2 (\bW_t - \bW_s)^2.
    \end{align*}
    Taking $\mathbb{E}[\,\cdot\mid \bX_s]$ and using $\mathbb{E}[\bW_t - \bW_s \mid \bX_s] = 0$ and $\mathbb{E}[(\bW_t - \bW_s)^2 \mid \bX_s] = t - s$,
    \begin{equation*}
        \mathbb{E}[\mathcal{L}_\text{EM} \mid \bX_s] = (t-s)^2 \left(\bm f(s, \bX_s) - \bm f_{s,t}(\bX_s)\right)^2 + (t-s) \left(\bm g(s) - \bm g_{s,t}\right)^2.
    \end{equation*}
    Substituting for $\mathcal{L}_f, \mathcal{L}_g$ gives the result.
\end{proof}

\subsection{Diffusion SDEs}
\label{appsubsec:diffusion-sde}
To apply \cref{alg:ssfm} to diffusion SDEs, we must derive the ground truth reverse diffusion SDE. Consider the variance preserving formulation,
\begin{equation*}
    \rmd \bX_t = -\frac{1}{2}\beta_t \bX_t\; \rmd t + \sqrt{\beta_t}\; \rmd \bW_t.
\end{equation*}
The reverse diffusion SDE is given by
\begin{equation*}
    \rmd \bX_t = [-\frac{1}{2}\beta_t \bX_t - \beta_t \nabla_{\bX_t}\log p(t, \bX_t)]\; \rmd t + \sqrt{\beta_t} \;\rmd \bW_t,
\end{equation*}
where $\nabla_{\bX_t}\log p(t, \bX_t)$ is the score.

Given a data sample $\bX_1$, an expression for the score can be obtained. Specifically, we sample the forward process via $\bX_t = \alpha_t \bX_1 + \sigma_t \epsilon, \, \epsilon \sim \mathcal{N}(\bm 0, \bm I)$, where
\begin{equation*}
    \alpha_t = \exp\!\left(-\tfrac{1}{2}\int_0^t \beta_s\; ds\right), \quad \sigma_t^2 = 1-\alpha_t^2.
\end{equation*}
Then, the score is given by
\begin{equation*}
    \nabla_{\bX_t}\log p(t, \bX_t) = -\frac{\epsilon}{\sigma_t}.
\end{equation*}

We are therefore able to build the ground truth SDE required for training the stochastic flow map by
\begin{equation*}
    \rmd \bX_t = [-\frac{1}{2}\beta_t \bX_t + \beta_t \frac{\epsilon}{\sigma_t}]\; \rmd t + \sqrt{\beta_t}\; \rmd \bW_t.
\end{equation*}

\section{Extended related work}
\paragraph{Few-step models.}
The training objectives for SSFMs with small step sizes are related to generator matching, \ie, matching the drift and diffusion coefficients \parencite{holderrieth2025generator}.
The semigroup loss is the stochastic analogue to the one used in deterministic flow maps studied by \parencite{boffi2025build,frans2025one}.
Consistency models \parencite{song2023consistency,song2024improved} use a small Euler step which is related to our small stochastic Euler-Maruyama step; however, the actual loss itself is quite different as we predict the small jump rather than using consistency.
Other methods \parencite{boffi2025build,geng2025mean,geng2025improved} train with losses which require time derivatives of the flow map, a requirement that does not extend naturally to the stochastic setting, as the Brownian motion is nowhere differentiable \wrt time.

\paragraph{Few-step sampling of SDEs.}
An active area of research has been accelerating inference with diffusion SDEs via more efficient numerical schemes \parencite{zhang2023fast,gonzalez2023seeds,blasingame2026rex}, thereby decreasing the NFE to use these models; however, all these techniques are at inference time and use the same models.
Recent work by \textcite{jiang2025sig} looks at using \textit{partial signatures} of the Brownian motion to distill pre-trained diffusion SDEs into taking larger step sizes by learning a hyper-solver \parencite{NEURIPS2020_f1686b4b}.
This is significantly different from our work as it requires numerically integrating the entire SDE, \ie, it does not have a simulation-free manner of training even for flow/diffusion models and relies on a pre-trained diffusion model teacher.
SSFMs on the other hand enable scalable training of stochastic flow maps and can be trained without a pre-trained teacher model.

\paragraph{Stochastic few-step models.}
Recent work by \textcite{kiyohara2025neural,passaro2026stochastic,holderrieth2026diamond,potaptchik2026meta} have looked at weak approximations to the diffusion SDE.
The later three works \textcite{passaro2026stochastic,holderrieth2026diamond,potaptchik2026meta} learn an ``inner'' flow map to learn the transition kernel as we detail in the main paper.\footnote{\textcite{passaro2026stochastic} discusses the nuances between these works in more detail.}
Beyond providing a strong solution, SSFMs also work for arbitrary additive-noise SDEs, whereas existing weak approaches are typically formulated for specific diffusion model SDEs.
The work of \textcite{kiyohara2025neural} also adopts an Euler-Maruyama type parametrization for each layer within the context discrete normalizing flow \parencite{rezende2015nfs} framework, \ie, for a flow map $\bm \Upsilon_{s,t}^\theta(\bx) = (\bmf_L^\theta \circ \bmf_{L-1}^\theta \circ \cdots \circ \bmf_1^\theta)(\bx)$, each layer $\bmf_\ell$ is given as
\begin{equation}
    \bmf_\ell(\bx; s, t) = \bx + (t-s) \bm \mu_\ell^\theta(\bx; s, t) + \sqrt{t-s} \bm \sigma_\ell^\theta(\bx; s, t) \odot \bm\varepsilon, \quad \bm\varepsilon \sim \mathcal N(\bm 0, \bm I).
\end{equation}
This map based on the normalizing flow above is then trained to learn the transition kernel of an SDE $p_{t|s}(\bx_t|\bx_s)$ by minimizing both the forward and reverse Kullback-Leibler divergence \parencite{kullback1951information}.
This work is different from the other stream of weak stochastic flow maps as they learn via a parameterized normalizing flow instead of the standard flow map losses \parencite{geng2025improved,boffi2025build}; the paper focuses on learning the solution maps for latent neural SDEs \parencite{kidger2021efficient,li2020sdes}.

\section{Experimental details}
\label{app:experiments}
In the image and molecule generation experiments we take the number of polynomial coefficients $N=3$. We use the Virtual Brownian Tree (VBT) implementation in Diffrax to generate the coefficients $\bm I_{s, t}^{(3)}$ \parencite{jelinvcivc2024single, kidger_thesis}.
It should be noted that the VBT implementation introduces a constant scaling factor on each coefficient.

\begin{figure}[h]
    \centering
    \includegraphics[width=\linewidth]{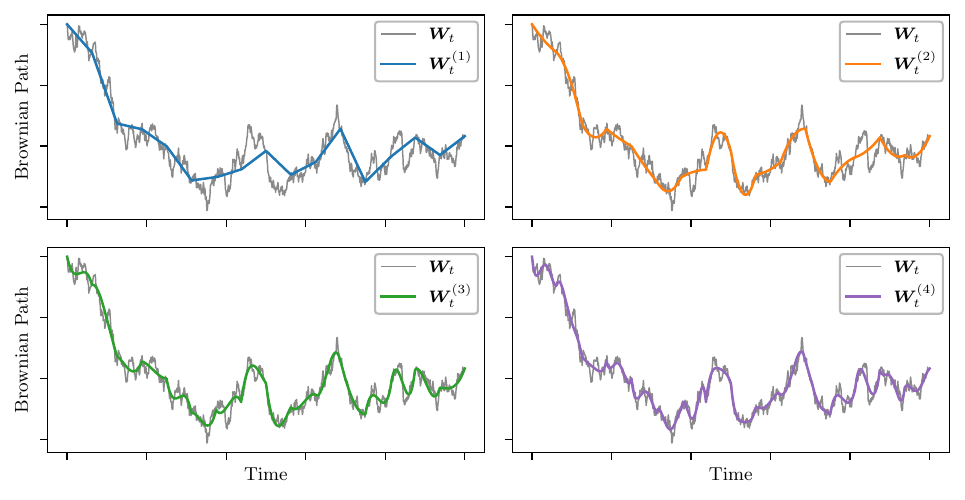}
    \caption{Brownian path $\bW_t$ and associated polynomial approximations $\bW_t^{(N)}$ over 16-steps.}
    \label{fig:polyapprox_many}
\end{figure}

\subsection{Non-linear SDE}
\label{appsubsec:non_linear}
We include additional informative plots for the non-linear SDE experiment that would not fit in the main text. In \cref{fig:polyapprox_many} we show the polynomial approximation of Brownian motion for varying degree polynomials. In \cref{fig:non_linear_many} we show the SSFM predictions for many step sizes when trained with the $N=4$ degree polynomial. 

\begin{figure}
    \centering
    \includegraphics[width=\linewidth]{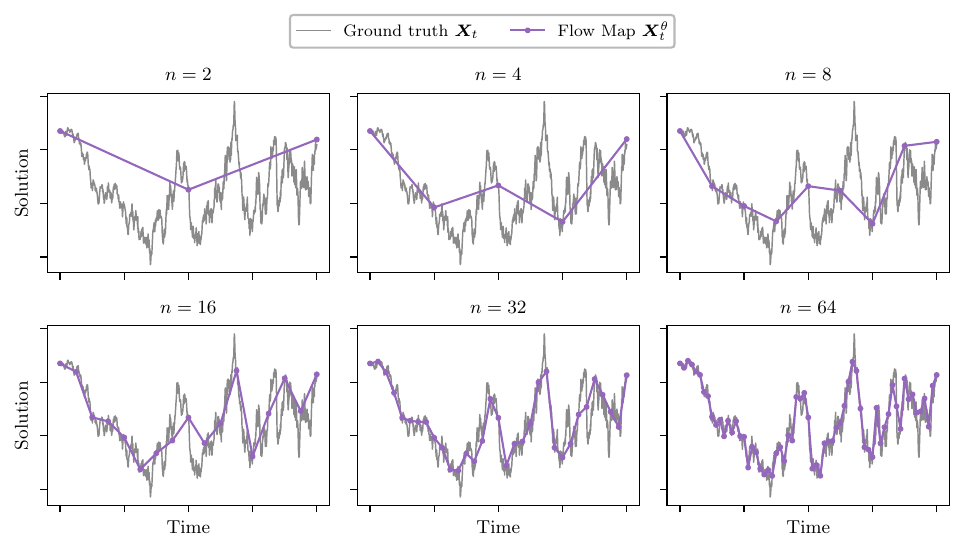}
    \caption{Ground truth SDE and SSFM prediction for many step counts.}
    \label{fig:non_linear_many}
\end{figure}

\subsection{Image generation}
For the image generation experiments we setup the variance preserving SDE (see \cref{appsubsec:diffusion-sde}) as the ground truth and follow \cref{alg:ssfm} to train the SSFM. The drift and diffusion networks are parameterized with the EDM2 architecture (Config C) \parencite{karras2024analyzing} with hyperparameters given in \cref{tab:image-hyperparams}.

For evaluation, we generate 50k samples with the SSFM and compute FID against the CIFAR-10 / CelebA datasets. Uniform step placement is used so that for $n$ steps, the step size is given by $(1 - t_\epsilon)/N$, where $t_\epsilon=10^{-5}$.

\begin{table}[h]
  \centering
  \caption{Image generation hyperparameters.}
  \label{tab:image-hyperparams}
  \begin{tabular}{lrr}
  \toprule
  Hyperparameter & CIFAR-10 & CelebA-64 \\
  \midrule
  \multicolumn{2}{l}{\emph{Drift network (EDM2 U-Net)}} \\
  Base channels               & 128 & 128 \\
  Channel multipliers         & $(2,\,2,\,2)$ & $(1, 2, 3, 4)$ \\
  Attention resolutions       & $16\times 16$ & $16\times 16, 8 \times 8$ \\
  Attention head dimension    & 64 & 64 \\
  GroupNorm groups            & 8 & 8 \\
  Dropout                     & 0.13 & 0.13 \\
  \midrule
  \multicolumn{2}{l}{\emph{Diffusion network (EDM2 U-Net)}} \\
  Base channels               & 64 & 64 \\
  Channel multipliers         & $(2,\,2,\,2)$ & $(1, 2, 3, 4)$ \\
  Attention resolutions       & $16\times 16$ & $16\times 16, 8 \times 8$ \\
  Attention head dimension    & 64 & 64 \\
  GroupNorm groups            & 8 & 8 \\
  Dropout                     & 0 & 0 \\
  \midrule
  \multicolumn{2}{l}{\emph{Loss}} \\
  Step size split $\Delta t$  & 0.01 & 0.01 \\
  Max distillation step $h_{\max}$ & 0.52 & 0.52 \\
  Batch split $\eta$       & 0.75 & 0.75 \\
  \midrule
  \multicolumn{2}{l}{\emph{Optimization}} \\
  Optimizer                   & Adam ($\beta_2 = 0.99$) & Adam ($\beta_2 = 0.99$) \\
  Peak learning rate               & $10^{-3}$ & $10^{-3}$ \\
  Min learning rate                 & $10^{-5}$ & $10^{-5}$ \\
  LR schedule                 & warmup + cosine decay & warmup + cosine decay \\
  Warmup steps                      & 5{,}000 & 5{,}000\\
  Gradient clip (global norm)       & 1.0 & 1.0 \\
  Batch size                  & 512 & 256 \\
  Training steps              & 400{,}000 & 800{,}000 \\
  EMA decay                   & 0.999 & 0.999 \\
  \bottomrule
  \end{tabular}
  \end{table}

\subsection{Molecular systems: ALDP}
\label{appsubsec:molecules}

The alanine dipeptide dataset follows \parencite{plainer2026consistent} with 50k samples from a molecular dynamics simulation in implicit solvent \parencite{kohler2021smooth}, coarse grained to five atoms [C, N, CA, C, N]. It is available from the \emph{ScoreMD} repository accompanying \parencite{plainer2026consistent}. 

The evaluation metrics calculated are the potential of mean force (PMF) squared error and the Jensen-Shannon (JS) divergence. These metrics both compare the difference in the equilibrium free energy surfaces of the ground truth system and the model prediction.

A graph transformer architecture is used with hyperparameters listed in \cref{tab:aldp-hparams}. The variance preserving SDE is used as the ground truth to construct the SSFM following \cref{alg:ssfm}.

The diffusion baseline models were obtained from the \emph{ScoreMD} repository and results collected via the evaluation code provided. The SSFM model evaluation was also computed from the \emph{ScoreMD} code to ensure a fair comparison.

\begin{table}[h]
  \centering
  \caption{ALDP generation hyperparameters.}
  \begin{tabular}{lr}
  \toprule
  Hyperparameter & ALDP \\
  \midrule
  \multicolumn{2}{l}{\textit{Drift/Diffusion network (graph-transformer)}} \\
  Hidden dim                        & 96 \\
  Transformer blocks                & 3/2 \\
  Attention heads                   & 8 \\
  Head dim                          & 64 \\
  Feed-forward multiplier           & 4 \\
  Time-embedding dim                & 64 \\
  Uncertainty MLP Fourier dim       & 64 \\
  \midrule
  \multicolumn{2}{l}{\textit{Loss}} \\
  Step size split $\Delta t$              & $10^{-3}$ \\
  Max distillation step $h_{\max}$  & 0.52 \\
  Batch split $\eta$        & 0.75 \\
  \midrule
  \multicolumn{2}{l}{\textit{Optimization}} \\
  Optimizer                         & AdamW \\
  Peak learning rate                & $10^{-3}$ \\
  Min learning rate                 & $10^{-5}$ \\
  LR schedule                       & warmup + cosine decay \\
  Warmup steps                      & 1{,}000 \\
  Gradient clip (global norm)       & 10 \\
  Batch size                        & 1{,}024 \\
  Training steps                    & 400{,}000 \\
  EMA decay                         & 0.999 \\
  \bottomrule
  \end{tabular}
  \label{tab:aldp-hparams}
  \end{table}

\subsection{Molecular systems: Chignolin}
\label{appsubsec:molecules_chignolin}

Chignolin, a $10$-residue mini-protein with sequence
\texttt{GYDPETGTWG}, was modelled in full atomic detail ($138$ atoms, all
hydrogens included). The reference ensemble is drawn from a still-internal addendum of the
\emph{many-peptides} molecular-dynamics dataset: long all-atom MD
trajectories that have been subsampled to $N = 10{,}000$ approximately
decorrelated conformers, which constitute the entire reference set used
below. Coordinates are stored in nanometres and are used without
re-centering.

We report two complementary Wasserstein distances, both estimated against the reference set with $10{,}000$ samples. The first
is the backbone torsion torus-Wasserstein distance $\T\text{-W}_2$. For
each conformer we extract the $9$ backbone $(\phi, \psi)$ torsion pairs.

The second metric is a tICA-$\text{W}_2$. We project conformers
into the dataset's precomputed $2$-dimensional time-lagged independent
component (tICA) coordinates -- built from $\mathrm{C}/\mathrm{N}/\mathrm{S}$
pairwise distances together with backbone $\phi/\psi/\omega$ sine/cosine
features ($1765$ features reduced to the two slowest components) -- and
compute the exact $2$D $\text{W}_2$ between generated and reference tICA
densities.

We compare a strong stochastic flow map (SSFM) against a diffusion
baseline, both trained on identical data with matched optimization. The
SSFM instantiates \cref{alg:ssfm} over a variance-preserving (VP) SDE. The drift and diffusion networks are parameterised as 
Diffusion Transformer (DiT) networks (see \cref{tab:chignolin-hparams} for hyperparameters). The Brownian polynomial degree is $N=3$.

The baseline is an x-prediction diffusion model: a single DiT trained
to denoise under the same VP SDE and sampled with ancestral VP reverse
steps. All transformer backbones are DiTs with adaLN-zero modulation,
sinusoidal time embeddings, and per-atom learned positional embeddings.

Additional plotted results showing the pooled Ramachandran and tICA projection are given in Figures \ref{fig:chignolin-pooled-ramachandran} and \ref{fig:chignolin-pooled-tica}.

\begin{table}[h]
  \centering
  \caption{Chignolin generation hyperparameters.}
  \label{tab:chignolin-hparams}
  \begin{tabular}{lrr}
    \toprule
     & SSFM & Diffusion \\
    \midrule
    \multicolumn{3}{l}{\textit{Data}} \\
    Sequence & \multicolumn{2}{c}{\texttt{GYDPETGTWG}} \\
    Atoms & \multicolumn{2}{c}{$138$} \\
    Reference frames & \multicolumn{2}{c}{$10{,}000$} \\
    \midrule
    \multicolumn{3}{l}{\textit{Drift/Diffusion network (DiT)}} \\
    Hidden size & $192$ & $192$ \\
    Transformer blocks & $6/2$ & $6$ \\
    Attention heads & $6$ & $6$ \\
    MLP ratio & $4$ & $4$ \\
    Time Fourier dim & $256$ & $256$ \\
    Input channels & $12$ & $3$ \\
    Dropout & $0.0$ & $0.0$ \\
    \midrule
    \multicolumn{3}{l}{\textit{Loss}} \\
    Step size split$\Delta t$ & $10^{-3}$ & --- \\
    Max distillation step $h_{\max}$ & $1.0$ & --- \\
    Batch split $\eta$ & $0.75$ & --- \\
    \midrule
    \multicolumn{3}{l}{\textit{Optimization}} \\
    Optimizer & AdamW & AdamW \\
    Peak learning rate & $2\times10^{-4}$ & $2\times10^{-4}$ \\
    Min learning rate & $10^{-5}$ & $10^{-5}$ \\
    LR schedule & \multicolumn{2}{c}{warmup + cosine decay} \\
    Warmup steps & $5000$ & $5000$ \\
    Weight decay & $10^{-4}$ & $10^{-4}$ \\
    Gradient clip (global norm) & $1.0$ & $1.0$ \\
    Batch size & $64$ & $64$ \\
    Training steps & $5\times10^{5}$ & $5\times10^{5}$ \\
    EMA decay & $0.999$ & --- \\
    \midrule
    Parameters & $5.71$M & $4.19$M \\
    \bottomrule
  \end{tabular}
\end{table}

\begin{figure}
    \centering
    \includegraphics[width=\linewidth]{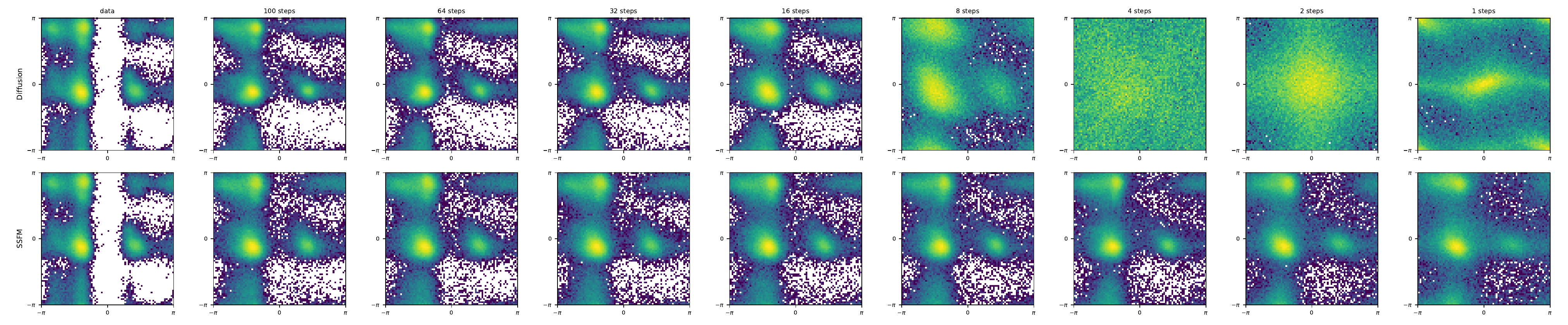}
    \caption{Pooled Chignolin Ramachandran plots. All backbone $\phi/\psi$ torsions are superimposed for ground-truth data (\emph{left-most column}) and generated samples from the diffusion baseline (\emph{top}) and SSFM (\emph{bottom}) across sampling step counts: [100, 64, 32, 16, 8, 4, 2, 1].}
    \label{fig:chignolin-pooled-ramachandran}
\end{figure}

\begin{figure}
    \centering
    \includegraphics[width=\linewidth]{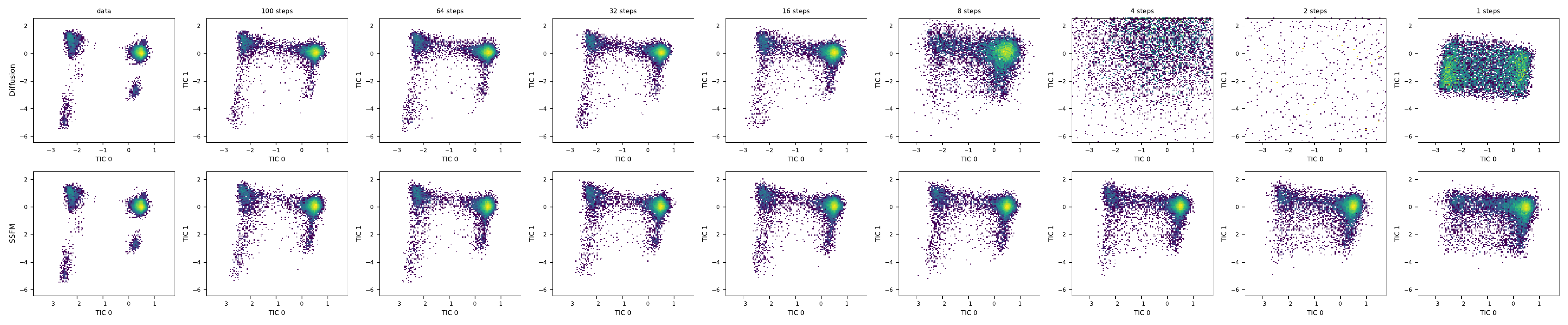}
    \caption{Chignolin tICA plots. Ground-truth data is shown in the \emph{left-most column}. The two primary tICA components and projected samples are shown from the diffusion baseline (\emph{top}) and SSFM (\emph{bottom}) across sampling step counts: [100, 64, 32, 16, 8, 4, 2, 1].}
    \label{fig:chignolin-pooled-tica}
\end{figure}

\newpage

\subsection{Hardware}
\label{app:compute}
All experiments were run on either one/two NVIDIA RTX A6000 GPUs or one/two NVIDIA H100 GPUs.

\subsection{Repositories}
\label{app:repos}
We made use of the following repositories and resources:
\begin{enumerate}
    \item \href{https://github.com/patrick-kidger/diffrax}{\texttt{patrick-kidger/diffrax}} (for VBT)
    \item \href{https://github.com/patrick-kidger/equinox}{\texttt{patrick-kidger/equinox}} (for neural networks in JAX)
    \item \href{https://github.com/Lightning-AI/torchmetrics}{\texttt{Lightning-AI/torchmetrics}} (for FID)
    \item \href{https://github.com/noegroup/ScoreMD}{\texttt{noegroup/ScoreMD}} (for Alanine-Dipeptide dataset and diffusion baselines)
\end{enumerate}

\section{Discussions}
\subsection{Contribution statement}
SM and ZB initially conceived the idea and developed it over correspondence with NR and JF.
The development of the theory was led by SM. The rough path convergence of the Brownian polynomial was shown by TH.
The image experiments and related open-source code were led by SM and ZB.
The molecular experiments and related open-source code were led by SM (ALDP) and NR (Chignolin).
SM and ZB wrote the paper, with contributions from all other authors. AT and JF guided the project.

\subsection{Broader impacts}
\label{app:broader}
We propose a framework for few-step generation of additive-noise SDEs, with demonstrated applications in image generation and molecular dynamics.
The primary positive impacts of this work is the acceleration of molecular dynamic simulations relevalent to drug discovery and other ai4science applications which use diffusion models. Improvements to image generation efficiency also reduce inference-time energy consumption. As with all advances in generative modelling, there is a risk that improved image generation could be misused to produce harmful synthetic media; however, SSFMs represent an efficiency improvement to existing pipelines rather than a new capability, and we do not believe this work introduces risks beyond those already present in deployed diffusion models.

\subsection{Limitations}
\label{app:limitations}
The current framework was demonstrated on additive-noise SDEs; extending this framework to state-dependent SDEs is possible with the theory developed here. We leave this to future work.
The polynomial approximation of the Brownian path introduces a truncation 
error controlled by the degree $N$, which in practice requires tuning as a 
hyperparameter.
Our image generation experiments are conducted on CIFAR-10 and CelebA-64; 
scaling to larger datasets and higher resolutions remains to be 
demonstrated.
Finally, the empirical comparison to weak stochastic flow maps focuses 
primarily on generative performance metrics; a more detailed empirical 
study of the pathwise consistency properties and their downstream implications is left to future work.

\end{document}

%% file: tikz/sfm_s_to_t.tex
  \begin{tikzpicture}[font=\small, >=stealth, line cap=round]

  \def\xs{1.2}
  \def\xu{5.5}
  \def\xt{10.8}
  \def\yPA{3.80}   
  \def\yPB{0.30}   
  \def\NSTEPS{200}
  \def\DXfull{0.0475}
  \def\driftStep{0.004}   

  \fill[tufte-shade,rounded corners=4pt] (0.0,2.20) rectangle (12.0,5.50);
  \fill[tufte-shade,rounded corners=4pt] (0.0,-1.50) rectangle (12.0,2.10);

  \draw[->,thick,tufte-dark] (0.7,2.12)--(11.4,2.12);
  \foreach \xx/\lb in {\xt/t}{
      \draw[tufte-rule,thin](\xx,2.02)--(\xx,2.22);
      \node[font=\footnotesize,tufte-mid] at (\xx,1.88) {$\lb$};
  }

  \node[font=\small\bfseries,tufte-dark,anchor=west]
      at (0.20,5.32)
      {Strong SFM (Ours)};
  \node[font=\small\bfseries,tufte-dark,anchor=west]
      at (0.20,1.92)
      {Weak SFM (Diamond / MFM)};


  \pgfmathsetseed{42}
  \coordinate (TRUE0) at (\xs,\yPA);
  \pgfmathsetmacro{\cY}{\yPA}\xdef\cYg{\cY}
  \foreach \k in {1,...,\NSTEPS}{
      \pgfmathsetmacro{\nY}{\cYg+\driftStep+0.055*rand}
      \xdef\cYg{\nY}
      \coordinate (TRUE\k) at ({\xs+\k*\DXfull},\nY);
  }
  \coordinate (TRUEu) at (TRUE90);   


  \begin{scope}\clip(0.0,2.20)rectangle(12.0,5.50);
      \foreach \seed in {11,23,37,53,67,79,83,97}{
          \pgfmathsetseed{\seed}
          \draw[tufte-red,line width=0.5pt,opacity=0.16]
              (\xs,\yPA)
              \foreach \k in {1,...,\NSTEPS}{
                  --++(\DXfull,\driftStep+0.055*rand)};
      }
      \foreach \seed in {13,19,31,41,59,61,71,73}{
          \pgfmathsetseed{\seed}
          \draw[tufte-red,line width=0.5pt,opacity=0.16]
              (\xs,\yPA)
              \foreach \k in {1,...,\NSTEPS}{
                  --++(\DXfull,-0.0058+0.055*rand)};
      }
      \draw[tufte-red,line width=1.9pt]
          (TRUE0)\foreach \k in {1,...,\NSTEPS}{--(TRUE\k)};
  \end{scope}

  \draw[tufte-dark,line width=1.6pt,->]
      (\xs,\yPA) -- (TRUEu)
      node[midway,above=14pt,font=\footnotesize,fill=tufte-shade,
           inner sep=2pt]{$\bm\Psi_{s,u}(\bx_s, \omega)$};
  \draw[tufte-dark,line width=1.6pt,->]
      (TRUEu) -- (TRUE200)
      node[midway,above=14pt,font=\footnotesize,fill=tufte-shade,
           inner sep=2pt]{$\bm\Psi_{u,t}(\bx_u, \omega)$};

  \draw[tufte-dark,line width=1.0pt,->,opacity=0.45]
      (\xs,\yPA) .. controls +(3.0,-0.70) and +(-3.0,-0.50) .. (TRUE200)
      node[midway,below=4pt,font=\footnotesize,fill=tufte-shade,
           inner sep=2pt]{$\bm\Psi_{s,t}(\bx_s,\omega)$};

  \node[draw=tufte-rule,fill=white,rounded corners=2pt,
        font=\footnotesize,inner sep=4pt,anchor=south east]
      at (11.80,2.28)
      {same $\omega$ $\Rightarrow$ same $\bx_t$ \checkmark};

  \filldraw[tufte-dark] (\xs,\yPA) circle(2.3pt)
      node[left,font=\footnotesize]{$\bx_s$};
  \filldraw[tufte-dark] (TRUEu) circle(2.3pt)
      node[below right=2pt and 1pt,font=\footnotesize,
           fill=tufte-shade,inner sep=2pt]{$\bx_u$};
  \filldraw[tufte-dark] (TRUE200) circle(2.8pt)
      node[right,font=\footnotesize]{$\bx_t$};


  \begin{scope}\clip(0.0,-1.50)rectangle(12.0,2.10);
      \foreach \seed in {11,23,37,53,67,79,83,97}{
          \pgfmathsetseed{\seed}
          \draw[tufte-red,line width=0.5pt,opacity=0.16]
              (\xs,\yPB)
              \foreach \k in {1,...,\NSTEPS}{
                  --++(\DXfull,\driftStep+0.055*rand)};
      }
      \foreach \seed in {13,19,31,41,59,61,71,73}{
          \pgfmathsetseed{\seed}
          \draw[tufte-red,line width=0.5pt,opacity=0.16]
              (\xs,\yPB)
              \foreach \k in {1,...,\NSTEPS}{
                  --++(\DXfull,-0.0058+0.055*rand)};
      }
      \pgfmathsetseed{42}
      \draw[tufte-red,line width=1.9pt]
          (\xs,\yPB)
          \foreach \k in {1,...,\NSTEPS}{
              --++(\DXfull,\driftStep+0.055*rand)};
  \end{scope}

  \coordinate (Ws)     at (\xs,   \yPB);
  \coordinate (WbXu)   at (\xu,   \yPB-0.70);  
  \coordinate (WbXt2)  at (\xt,   \yPB+0.60);  

  \coordinate (WbXt1)  at (\xt,   \yPB+1.10);  

  \draw[tufte-dark,line width=1.3pt,dashed,->]
      (Ws) -- (WbXu)
      node[midway,below=8pt,font=\footnotesize,fill=tufte-shade,
           inner sep=2pt]{$\bar{\bm\Psi}_{s,u}(\bx_s,\bm\epsilon_1)$};
  \draw[tufte-dark,line width=1.3pt,dashed,->]
      (WbXu) -- (WbXt2)
      node[midway,below=12pt,font=\footnotesize,fill=tufte-shade,
           inner sep=2pt]{$\bar{\bm\Psi}_{u,t}(\bar\bx_u,\bm\epsilon_2)$};

  \draw[tufte-dark,line width=1.0pt,dashed,->,opacity=0.55]
      (Ws) .. controls +(5.0,1.20) and +(-3.0,0.20) .. (WbXt1)
      node[midway,above=4pt,font=\footnotesize,fill=tufte-shade,
           inner sep=2pt]{$\bar{\bm\Psi}_{s,t}(\bx_s,\bm\epsilon_3)$};

  \filldraw[tufte-dark]     (Ws)    circle(2.3pt)
      node[left,font=\footnotesize]{$\bx_s$};
  \filldraw[tufte-dark] (WbXu)  circle(2.3pt)
      node[below right=2pt and 1pt,font=\footnotesize,
           fill=tufte-shade,inner sep=2pt]{$\bar\bx_u$};
  \filldraw[tufte-dark] (WbXt2) circle(2.8pt)
      node[right,font=\footnotesize]
          {$\bar\bx_t^{\,(2)}$};
  \filldraw[tufte-dark!70] (WbXt1) circle(2.8pt)
      node[right,font=\footnotesize,tufte-dark]
          {$\bar\bx_t^{\,(1)}$};

  \node[draw=tufte-rule,fill=white,rounded corners=2pt,
        font=\footnotesize,inner sep=4pt,anchor=south east]
      at (11.80,-1.44)
      {$\bar\bx_t^{(1)} \neq \bar\bx_t^{(2)}$ in general};

  \node[anchor=base,font=\footnotesize] at (6.0,-1.72)
      {\textcolor{tufte-red}{\rule[0.4ex]{10pt}{1.6pt}}\;$\bX_t(\omega)$
       \enspace
       \textcolor{tufte-red!40}{\rule[0.4ex]{10pt}{0.55pt}}\;other realizations
       \enspace
       {\color{tufte-dark}\rule[0.4ex]{10pt}{1.4pt}}\;strong (ours)
       \enspace
       \tikz[baseline=-0.5ex]{
           \draw[tufte-dark,line width=1.1pt,dashed]
               (0,0)--(0.38,0);}\;weak};

\end{tikzpicture}

%% file: bib.bib
@inproceedings{
holderrieth2026glass,
title={{GLASS} Flows: Efficient Inference for Reward Alignment of Flow and Diffusion Models},
author={Peter Holderrieth and Uriel Singer and Tommi Jaakkola and Ricky T. Q. Chen and Yaron Lipman and Brian Karrer},
booktitle={The Fourteenth International Conference on Learning Representations},
year={2026},
url={https://openreview.net/forum?id=vH7OAPZ2dR}
}

@inproceedings{
lipman2023flow,
title={Flow Matching for Generative Modeling},
author={Yaron Lipman and Ricky T. Q. Chen and Heli Ben-Hamu and Maximilian Nickel and Matthew Le},
booktitle={The Eleventh International Conference on Learning Representations },
year={2023},
url={https://openreview.net/forum?id=PqvMRDCJT9t}
}

@article{albergo2025stochastic,
  title={Stochastic interpolants: A unifying framework for flows and diffusions},
  author={Albergo, Michael and Boffi, Nicholas M and Vanden-Eijnden, Eric},
  journal={Journal of Machine Learning Research},
  volume={26},
  number={209},
  pages={1--80},
  year={2025}
}

@inproceedings{
liu2023flow,
title={Flow Straight and Fast: Learning to Generate and Transfer Data with Rectified Flow},
author={Xingchao Liu and Chengyue Gong and Qiang Liu},
booktitle={The Eleventh International Conference on Learning Representations },
year={2023},
url={https://openreview.net/forum?id=XVjTT1nw5z}
}

@article{
tong2024improving,
title={Improving and generalizing flow-based generative models with minibatch optimal transport},
author={Alexander Tong and Kilian FATRAS and Nikolay Malkin and Guillaume Huguet and Yanlei Zhang and Jarrid Rector-Brooks and Guy Wolf and Yoshua Bengio},
journal={Transactions on Machine Learning Research},
issn={2835-8856},
year={2024},
url={https://openreview.net/forum?id=CD9Snc73AW},
note={Expert Certification}
}

@misc{peluchetti2021,
      title={Non-Denoising Forward-Time Diffusions}, 
      author={Stefano Peluchetti},
      year={2021},
}

@misc{lipman2024flow-guide,
  title={Flow Matching Guide and Code},
  author={Lipman, Yaron and Havasi, Marton and Holderrieth, Peter and Shaul, Neta and Le, Matt and Karrer, Brian and Chen, Ricky TQ and Lopez-Paz, David and Ben-Hamu, Heli and Gat, Itai},
  year={2024}
}

@inproceedings{
song2021scorebased,
title={Score-Based Generative Modeling through Stochastic Differential Equations},
author={Yang Song and Jascha Sohl-Dickstein and Diederik P Kingma and Abhishek Kumar and Stefano Ermon and Ben Poole},
booktitle={International Conference on Learning Representations},
year={2021},
url={https://openreview.net/forum?id=PxTIG12RRHS}
}

@article{anderson1982reverse,
  title={Reverse-time diffusion equation models},
  author={Anderson, Brian DO},
  journal={Stochastic Processes and their Applications},
  volume={12},
  number={3},
  pages={313--326},
  year={1982},
  publisher={Elsevier}
}

@article{maoutsa2020interacting,
  title={Interacting particle solutions of fokker--planck equations through gradient--log--density estimation},
  author={Maoutsa, Dimitra and Reich, Sebastian and Opper, Manfred},
  journal={Entropy},
  volume={22},
  number={8},
  pages={802},
  year={2020},
  publisher={MDPI}
}

@inproceedings{
nie2024blessing,
title={The Blessing of Randomness: {SDE} Beats {ODE} in General Diffusion-based Image Editing},
author={Shen Nie and Hanzhong Allan Guo and Cheng Lu and Yuhao Zhou and Chenyu Zheng and Chongxuan Li},
booktitle={The Twelfth International Conference on Learning Representations},
year={2024},
url={https://openreview.net/forum?id=DesYwmUG00}
}

@phdthesis{kidger_thesis,
    author = {Kidger, Patrick},
    title = {On Neural Differential Equations},
    school = {Oxford University},
    type = {Ph.D. thesis},
    year = {2022},
    note = {Available at \url{https://arxiv.org/abs/2202.02435}}
}

@misc{
heek2024multistep,
title={Multistep Consistency Models},
author={Jonathan Heek and Emiel Hoogeboom and Tim Salimans},
year={2024},
url={https://openreview.net/forum?id=d7DZRNe2xG}
}

@inproceedings{
song2024improved,
title={Improved Techniques for Training Consistency Models},
author={Yang Song and Prafulla Dhariwal},
booktitle={The Twelfth International Conference on Learning Representations},
year={2024},
url={https://openreview.net/forum?id=WNzy9bRDvG}
}

@inproceedings{
kim2024consistency,
title={Consistency Trajectory Models: Learning Probability Flow {ODE} Trajectory of Diffusion},
author={Dongjun Kim and Chieh-Hsin Lai and Wei-Hsiang Liao and Naoki Murata and Yuhta Takida and Toshimitsu Uesaka and Yutong He and Yuki Mitsufuji and Stefano Ermon},
booktitle={The Twelfth International Conference on Learning Representations},
year={2024},
url={https://openreview.net/forum?id=ymjI8feDTD}
}

@inproceedings{
frans2025one,
title={One Step Diffusion via Shortcut Models},
author={Kevin Frans and Danijar Hafner and Sergey Levine and Pieter Abbeel},
booktitle={The Thirteenth International Conference on Learning Representations},
year={2025},
url={https://openreview.net/forum?id=OlzB6LnXcS}
}

@article{boffi2025build,
  title={How to build a consistency model: Learning flow maps via self-distillation},
  author={Boffi, Nicholas M and Albergo, Michael S and Vanden-Eijnden, Eric},
  journal={arXiv preprint arXiv:2505.18825},
  year={2025}
}

@article{boffi2024flow,
  title={Flow map matching},
  author={Boffi, Nicholas M and Albergo, Michael S and Vanden-Eijnden, Eric},
  journal={arXiv preprint arXiv:2406.07507},
  year={2024}
}

@inproceedings{
sabour2025align,
title={Align Your Flow: Scaling Continuous-Time Flow Map Distillation},
author={Amirmojtaba Sabour and Sanja Fidler and Karsten Kreis},
booktitle={The Thirty-ninth Annual Conference on Neural Information Processing Systems},
year={2025},
url={https://openreview.net/forum?id=pzHuesCvcO}
}

@inproceedings{
geng2025mean,
title={Mean Flows for One-step Generative Modeling},
author={Zhengyang Geng and Mingyang Deng and Xingjian Bai and J Zico Kolter and Kaiming He},
booktitle={The Thirty-ninth Annual Conference on Neural Information Processing Systems},
year={2025},
url={https://openreview.net/forum?id=uWj4s7rMnR}
}

@InProceedings{song2023consistency,
  title = 	 {Consistency Models},
  author =       {Song, Yang and Dhariwal, Prafulla and Chen, Mark and Sutskever, Ilya},
  booktitle = 	 {Proceedings of the 40th International Conference on Machine Learning},
  pages = 	 {32211--32252},
  year = 	 {2023},
  editor = 	 {Krause, Andreas and Brunskill, Emma and Cho, Kyunghyun and Engelhardt, Barbara and Sabato, Sivan and Scarlett, Jonathan},
  volume = 	 {202},
  series = 	 {Proceedings of Machine Learning Research},
  month = 	 {23--29 Jul},
  publisher =    {PMLR},
  url = 	 {https://proceedings.mlr.press/v202/song23a.html}
}

@inproceedings{chen2018neural,
 author = {Chen, Ricky T. Q. and Rubanova, Yulia and Bettencourt, Jesse and Duvenaud, David K},
 booktitle = {Advances in Neural Information Processing Systems},
 editor = {S. Bengio and H. Wallach and H. Larochelle and K. Grauman and N. Cesa-Bianchi and R. Garnett},
 pages = {},
 publisher = {Curran Associates, Inc.},
 title = {Neural Ordinary Differential Equations},
 url = {https://proceedings.neurips.cc/paper_files/paper/2018/file/69386f6bb1dfed68692a24c8686939b9-Paper.pdf},
 volume = {31},
 year = {2018}
}

@book{friz2010multidimensional,
  title={Multidimensional stochastic processes as rough paths: theory and applications},
  author={Friz, Peter K and Victoir, Nicolas B},
  volume={120},
  year={2010},
  publisher={Cambridge University Press}
}

@book{oksendal2003stochastic,
  title={Stochastic Differential Equations: An Introduction with Applications},
  author={{\O}ksendal, Bernt},
  isbn={9783662036204},
  series={Universitext},
  doi={10.1007/978-3-642-14394-6},
  year={2003},
  month={jul},
  address={Berlin, Germany},
  publisher={Springer Berlin Heidelberg}
}

@article{lyons1998differential,
  title={Differential equations driven by rough signals},
  author={Lyons, Terry J},
  journal={Revista Matem{\'a}tica Iberoamericana},
  volume={14},
  number={2},
  pages={215--310},
  year={1998}
}

@article{foster2020optimal,
  title={An optimal polynomial approximation of Brownian motion},
  author={Foster, James and Lyons, Terry and Oberhauser, Harald},
  journal={SIAM Journal on Numerical Analysis},
  volume={58},
  number={3},
  pages={1393--1421},
  year={2020},
  publisher={SIAM}
}

@article{jelinvcivc2024single,
  title={Single-seed generation of Brownian paths and integrals for adaptive and high order SDE solvers},
  author={Jelin{\v{c}}i{\v{c}}, Andra{\v{z}} and Foster, James and Kidger, Patrick},
  journal={arXiv preprint arXiv:2405.06464},
  year={2024}
}

@article{potaptchik2026meta,
  title={Meta Flow Maps enable scalable reward alignment},
  author={Potaptchik, Peter and Saravanan, Adhi and Mammadov, Abbas and Prat, Alvaro and Albergo, Michael S and Teh, Yee Whye},
  journal={arXiv preprint arXiv:2601.14430},
  year={2026}
}

@inproceedings{
passaro2026stochastic,
title={Stochastic Few-step Models},
author={Romeo Passaro and Zander W. Blasingame and Michael M. Bronstein and Alexander Tong},
booktitle={ICLR 2026 2nd Workshop on Deep Generative Model in Machine Learning: Theory, Principle and Efficacy},
year={2026},
url={https://openreview.net/forum?id=nmczKNW73P}
}

@article{holderrieth2026diamond,
  title={Diamond Maps: Efficient Reward Alignment via Stochastic Flow Maps},
  author={Holderrieth, Peter and Chen, Douglas and Eyring, Luca and Shah, Ishin and Anantharaman, Giri and He, Yutong and Akata, Zeynep and Jaakkola, Tommi and Boffi, Nicholas Matthew and Simchowitz, Max},
  journal={arXiv preprint arXiv:2602.05993},
  year={2026}
}

@inproceedings{lu2022dpmsolver,
title={{DPM}-Solver: A Fast {ODE} Solver for Diffusion Probabilistic Model Sampling in Around 10 Steps},
author={Cheng Lu and Yuhao Zhou and Fan Bao and Jianfei Chen and Chongxuan Li and Jun Zhu},
booktitle={Advances in Neural Information Processing Systems},
editor={Alice H. Oh and Alekh Agarwal and Danielle Belgrave and Kyunghyun Cho},
year={2022},
url={https://openreview.net/forum?id=2uAaGwlP_V}
}

@article{chen1954iterated,
  title={Iterated integrals and exponential homomorphisms},
  author={Chen, Kuo-Tsai},
  journal={Proceedings of the London Mathematical Society},
  volume={3},
  number={1},
  pages={502--512},
  year={1954},
  publisher={Oxford University Press}
}

@article{chen1957integration,
  title={Integration of paths, geometric invariants and a generalized Baker-Hausdorff formula},
  author={Chen, Kuo-Tsai},
  journal={Annals of Mathematics},
  volume={65},
  number={1},
  pages={163--178},
  year={1957},
  publisher={JSTOR}
}

@inproceedings{sohl2015deep,
  title={Deep unsupervised learning using nonequilibrium thermodynamics},
  author={Sohl-Dickstein, Jascha and Weiss, Eric and Maheswaranathan, Niru and Ganguli, Surya},
  booktitle={International conference on machine learning},
  pages={2256--2265},
  year={2015},
  organization={pmlr}
}

@inproceedings{ho2020diffusion,
 author = {Ho, Jonathan and Jain, Ajay and Abbeel, Pieter},
 booktitle = {Advances in Neural Information Processing Systems},
 editor = {H. Larochelle and M. Ranzato and R. Hadsell and M.F. Balcan and H. Lin},
 pages = {6840--6851},
 publisher = {Curran Associates, Inc.},
 title = {Denoising Diffusion Probabilistic Models},
 url = {https://proceedings.neurips.cc/paper_files/paper/2020/file/4c5bcfec8584af0d967f1ab10179ca4b-Paper.pdf},
 volume = {33},
 year = {2020}
}

@inproceedings{rombach2022high,
  title={High-resolution image synthesis with latent diffusion models},
  author={Rombach, Robin and Blattmann, Andreas and Lorenz, Dominik and Esser, Patrick and Ommer, Bj{\"o}rn},
  booktitle={Proceedings of the IEEE/CVF conference on computer vision and pattern recognition},
  pages={10684--10695},
  year={2022}
}

@inproceedings{blattmann2023align,
  title={Align your latents: High-resolution video synthesis with latent diffusion models},
  author={Blattmann, Andreas and Rombach, Robin and Ling, Huan and Dockhorn, Tim and Kim, Seung Wook and Fidler, Sanja and Kreis, Karsten},
  booktitle={Proceedings of the IEEE/CVF conference on computer vision and pattern recognition},
  pages={22563--22575},
  year={2023}
}

@article{watson2023novo,
  title={De novo design of protein structure and function with RFdiffusion},
  author={Watson, Joseph L and Juergens, David and Bennett, Nathaniel R and Trippe, Brian L and Yim, Jason and Eisenach, Helen E and Ahern, Woody and Borst, Andrew J and Ragotte, Robert J and Milles, Lukas F and others},
  journal={Nature},
  volume={620},
  number={7976},
  pages={1089--1100},
  year={2023},
  publisher={Nature Publishing Group UK London}
}

@inproceedings{
rehman2026falcon,
title={{FALCON}: Few-step Accurate Likelihoods for Continuous Flows},
author={Rehman, Danyal and Akhound-Sadegh, Tara and Gazizov, Artem and Bengio, Yoshua and Tong, Alexander},
booktitle={The Fourteenth International Conference on Learning Representations},
year={2026},
url={https://openreview.net/forum?id=FbssShlI4N}
}

@inproceedings{plainer2026consistent,
  title={Consistent Sampling and Simulation: Molecular Dynamics with Energy-Based Diffusion Models},
  author={Plainer, Michael and Wu, Hao and Klein, Leon and G{\"u}nnemann, Stephan and Noe, Frank},
  booktitle={The Thirty-ninth Annual Conference on Neural Information Processing Systems},
  year={2025},
}

@inproceedings{
kidger2021efficient,
title={Efficient and Accurate Gradients for Neural {SDE}s},
author={Patrick Kidger and James Foster and Xuechen Li and Terry Lyons},
booktitle={Advances in Neural Information Processing Systems},
editor={A. Beygelzimer and Y. Dauphin and P. Liang and J. Wortman Vaughan},
year={2021},
url={https://openreview.net/forum?id=b2bkE0Qq8Ya}
}

@InProceedings{li2020sdes,
  title = 	 {Scalable Gradients and Variational Inference for
 Stochastic Differential Equations },
  author =       {Li, Xuechen and Wong, Ting-Kam Leonard and Chen, Ricky T. Q. and Duvenaud, David K.},
  booktitle = 	 {Proceedings of The 2nd Symposium on
 Advances in Approximate Bayesian Inference},
  pages = 	 {1--28},
  year = 	 {2020},
  editor = 	 {Zhang, Cheng and Ruiz, Francisco and Bui, Thang and Dieng, Adji Bousso and Liang, Dawen},
  volume = 	 {118},
  series = 	 {Proceedings of Machine Learning Research},
  month = 	 {08 Dec},
  publisher =    {PMLR},
  url = 	 {https://proceedings.mlr.press/v118/li20a.html}
}

@inproceedings{
oh2024stable,
title={Stable Neural Stochastic Differential Equations in Analyzing Irregular Time Series Data},
author={YongKyung Oh and Dongyoung Lim and Sungil Kim},
booktitle={The Twelfth International Conference on Learning Representations},
year={2024},
url={https://openreview.net/forum?id=4VIgNuQ1pY}
}

@inproceedings{
walker2024log,
title={Log Neural Controlled Differential Equations: The Lie Brackets Make A Difference},
author={Benjamin Walker and Andrew Donald McLeod and Tiexin Qin and Yichuan Cheng and Haoliang Li and Terry Lyons},
booktitle={Forty-first International Conference on Machine Learning},
year={2024},
url={https://openreview.net/forum?id=0tYrMtQyPT}
}

@article{disco2026,
      title={General Multimodal Protein Design Enables DNA-Encoding of Chemistry},
      author={Jarrid Rector-Brooks and Théophile Lambert and Marta Skreta and Daniel Roth and Yueming Long and Zi-Qi Li and Xi Zhang and Miruna Cretu and Francesca-Zhoufan Li and Tanvi Ganapathy and Emily Jin and Avishek Joey Bose and Jason Yang and Kirill Neklyudov and Yoshua Bengio and Alexander Tong and Frances H. Arnold and Cheng-Hao Liu},
      year={2026},
      eprint={2604.05181},
      archivePrefix={arXiv},
      primaryClass={cs.LG},
      url={https://arxiv.org/abs/2604.05181},
}

@inproceedings{
zhang2023fast,
title={Fast Sampling of Diffusion Models with Exponential Integrator},
author={Qinsheng Zhang and Yongxin Chen},
booktitle={The Eleventh International Conference on Learning Representations },
year={2023},
url={https://openreview.net/forum?id=Loek7hfb46P}
}

@inproceedings{
gonzalez2023seeds,
title={{SEEDS}: Exponential {SDE} Solvers for Fast High-Quality Sampling from Diffusion Models},
author={Martin Gonzalez and Nelson Fernandez and Thuy Vinh Dinh Tran and Elies Gherbi and Hatem Hajri and Nader Masmoudi},
booktitle={Thirty-seventh Conference on Neural Information Processing Systems},
year={2023},
url={https://openreview.net/forum?id=V6IgkYKD8P}
}

@inproceedings{blasingame2026rex,
  title = {Rex: A Family of Reversible Exponential (Stochastic) Runge-Kutta Solvers},
  author = {Blasingame, Zander W. and Liu, Chen},
  year = {2026},
  booktitle = {Forty-third International Conference on Machine Learning},
  url = {https://openreview.net/forum?id=7pQIzVNctu},
}

@article{geng2025improved,
  title={Improved mean flows: On the challenges of fastforward generative models},
  author={Geng, Zhengyang and Lu, Yiyang and Wu, Zongze and Shechtman, Eli and Kolter, J Zico and He, Kaiming},
  journal={arXiv preprint arXiv:2512.02012},
  year={2025}
}

@article{kohler2021smooth,
  title={Smooth normalizing flows},
  author={K{\"o}hler, Jonas and Kr{\"a}mer, Andreas and No{\'e}, Frank},
  journal={Advances in Neural Information Processing Systems},
  volume={34},
  pages={2796--2809},
  year={2021}
}

@article{krizhevsky2009learning,
  title={Learning multiple layers of features from tiny images},
  author={Krizhevsky, Alex and Hinton, Geoffrey and others},
  year={2009},
  publisher={Toronto, ON, Canada}
}

@inproceedings{liu2015faceattributes,
  title = {Deep Learning Face Attributes in the Wild},
  author = {Liu, Ziwei and Luo, Ping and Wang, Xiaogang and Tang, Xiaoou},
  booktitle = {Proceedings of International Conference on Computer Vision (ICCV)},
  month = {December},
  year = {2015} 
}

@inproceedings{karras2024analyzing,
  title={Analyzing and improving the training dynamics of diffusion models},
  author={Karras, Tero and Aittala, Miika and Lehtinen, Jaakko and Hellsten, Janne and Aila, Timo and Laine, Samuli},
  booktitle={Proceedings of the IEEE/CVF conference on computer vision and pattern recognition},
  pages={24174--24184},
  year={2024}
}

@article{garsia1970,
 ISSN = {00222518, 19435258},
 URL = {http://www.jstor.org/stable/24890119},
 author = {A. M. Garsia and E. Rodemich and H. Rumsey and M. Rosenblatt},
 journal = {Indiana University Mathematics Journal},
 number = {6},
 pages = {565--578},
 publisher = {Indiana University Mathematics Department},
 title = {A Real Variable Lemma and the Continuity of Paths of Some Gaussian Processes},
 urldate = {2026-05-06},
 volume = {20},
 year = {1970}
}

@article{Habermann_2021,
   title={A semicircle law and decorrelation phenomena for iterated Kolmogorov loops},
   volume={103},
   ISSN={1469-7750},
   url={http://dx.doi.org/10.1112/jlms.12384},
   DOI={10.1112/jlms.12384},
   number={2},
   journal={Journal of the London Mathematical Society},
   publisher={Wiley},
   author={Habermann, Karen},
   year={2021},
   month=Sept, pages={558–586} }

@inproceedings{
holderrieth2025generator,
title={Generator Matching: Generative modeling with arbitrary Markov processes},
author={Peter Holderrieth and Marton Havasi and Jason Yim and Neta Shaul and Itai Gat and Tommi Jaakkola and Brian Karrer and Ricky T. Q. Chen and Yaron Lipman},
booktitle={The Thirteenth International Conference on Learning Representations},
year={2025},
url={https://openreview.net/forum?id=RuP17cJtZo}
}

@article{jiang2025sig,
  title={Sig-DEG for Distillation: Making Diffusion Models Faster and Lighter},
  author={Jiang, Lei and Ge, Wen and Cariou-Kotlarek, Niels and Yi, Mingxuan and Chen, Po-Yu and Yang, Lingyi and Buet-Golfouse, Francois and Mittal, Gaurav and Ni, Hao},
  journal={arXiv preprint arXiv:2508.16939},
  year={2025}
}

@inproceedings{NEURIPS2020_f1686b4b,
 author = {Poli, Michael and Massaroli, Stefano and Yamashita, Atsushi and Asama, Hajime and Park, Jinkyoo},
 booktitle = {Advances in Neural Information Processing Systems},
 editor = {H. Larochelle and M. Ranzato and R. Hadsell and M.F. Balcan and H. Lin},
 pages = {21105--21117},
 publisher = {Curran Associates, Inc.},
 title = {Hypersolvers: Toward Fast Continuous-Depth Models},
 url = {https://proceedings.neurips.cc/paper_files/paper/2020/file/f1686b4badcf28d33ed632036c7ab0b8-Paper.pdf},
 volume = {33},
 year = {2020}
}

@inproceedings{
kiyohara2025neural,
title={Neural Stochastic Flows: Solver-Free Modelling and Inference for {SDE} Solutions},
author={Naoki Kiyohara and Edward Johns and Yingzhen Li},
booktitle={The Thirty-ninth Annual Conference on Neural Information Processing Systems},
year={2025},
url={https://openreview.net/forum?id=PrYDDxphym}
}

@InProceedings{rezende2015nfs,
  title = 	 {Variational Inference with Normalizing Flows},
  author = 	 {Rezende, Danilo and Mohamed, Shakir},
  booktitle = 	 {Proceedings of the 32nd International Conference on Machine Learning},
  pages = 	 {1530--1538},
  year = 	 {2015},
  editor = 	 {Bach, Francis and Blei, David},
  volume = 	 {37},
  series = 	 {Proceedings of Machine Learning Research},
  address = 	 {Lille, France},
  month = 	 {07--09 Jul},
  publisher =    {PMLR},
  url = 	 {https://proceedings.mlr.press/v37/rezende15.html}
}

@article{kullback1951information,
  title={On information and sufficiency},
  author={Kullback, Solomon and Leibler, Richard A},
  journal={The annals of mathematical statistics},
  volume={22},
  number={1},
  pages={79--86},
  year={1951},
  publisher={JSTOR}
}
